\def\year{2022}\relax
\documentclass[letterpaper]{article} 
\usepackage{aaai22}  
\usepackage{times}  
\usepackage{helvet}  
\usepackage{courier}  
\usepackage[hyphens]{url}  
\usepackage{graphicx} 
\usepackage{natbib}  
\usepackage{caption} 
\DeclareCaptionStyle{ruled}{labelfont=normalfont,labelsep=colon,strut=off} 
\usepackage{hyperref} 

\usepackage{tikz}
\usetikzlibrary{automata,positioning}
\usetikzlibrary{shapes, arrows, automata, positioning}
\usetikzlibrary{shapes.geometric}

\usetikzlibrary{arrows,decorations.pathmorphing,positioning,fit,trees,shapes,shadows,automata,calc} 
\usetikzlibrary{patterns,arrows,arrows.meta,calc,shapes,shadows,decorations.pathmorphing,decorations.pathreplacing,automata,shapes.multipart,positioning,shapes.geometric,fit,circuits,trees,shapes.gates.logic.US,fit, matrix,arrows.meta, quotes}
\usetikzlibrary{backgrounds,scopes}

\usepackage{pgfplots}
\usepackage[linesnumbered,ruled,vlined]{algorithm2e}

\usepackage{subcaption}
\usepackage{mathtools}
\usepackage{mathptmx}
\usepackage{bbm}
\usepackage{pbox}
\usepackage{xcolor,colortbl}

\usepackage[utf8]{inputenc} 
\usepackage{booktabs}       
\usepackage{amsfonts}       
\usepackage{amsmath}
\usepackage{nicefrac}       
\usepackage{microtype}      
\usepackage{amsthm}
\usepackage{bm}

\newtheorem{theorem}{Theorem}
\newtheorem{lemma}{Lemma}

\pgfplotsset{
    layers/my layer set/.define layer set={
        background,
        main,
        foreground
    }{
    },
    set layers=my layer set,
}

\title{Verifiable and Compositional Reinforcement Learning Systems}
\author {
    Cyrus Neary\textsuperscript{\rm 1},
    Christos Verginis\textsuperscript{\rm 2},
    Murat Cubuktepe\textsuperscript{\rm 3},
    Ufuk Topcu\textsuperscript{\rm 1,3}
}
\affiliations {
    \textsuperscript{\rm 1} Oden Institute for Computational Engineering and Sciences, The University of Texas at Austin, USA\\
    \textsuperscript{\rm 2} Division of Signals and Systems, Department of Electrical Engineering, Uppsala University, Sweden \\
    \textsuperscript{\rm 3} Department of Aerospace Engineering and Engineering Mechanics, The University of Texas at Austin, USA\\
    \{cneary, mcubuktepe, utopcu\}@utexas.edu, \; christos.verginis@angstrom.uu.se
}

\usepgfplotslibrary{external}
\pgfplotsset{compat=newest}
\newenvironment{customlegend}[1][]{%
    \begingroup
    \csname pgfplots@init@cleared@structures\endcsname
    \pgfplotsset{#1}%
}{%
    \csname pgfplots@createlegend\endcsname
    \endgroup
}%

\def\addlegendimage{\csname pgfplots@addlegendimage\endcsname}

\begin{document}

\newtheorem{manualtheoreminner}{Theorem}
\newenvironment{manualtheorem}[1]{%
  \renewcommand\themanualtheoreminner{#1}%
  \manualtheoreminner
}{\endmanualtheoreminner}

\maketitle

\newcommand{\defeq}{\vcentcolon=}

\setlength\marginparwidth{80pt}
\newcommand{\colorpar}[3]{\colorbox{#1}{\parbox{#2}{#3}}}
\newcommand{\marginremark}[3]{\marginpar{\colorpar{#2}{\linewidth}{\color{#1}#3}}}
\newcommand{\cn}[1]{\marginremark{red}{white}{\scriptsize{[CN]~ #1}}}
\newcommand{\mc}[1]{\marginremark{blue}{white}{\scriptsize{[MC]~ #1}}}
\newcommand{\cv}[1]{\marginremark{magenta}{white}{\scriptsize{[CV]~ #1}}}

\newcommand{\bernoulliProb}{p}
\newcommand{\rewardRV}{R}
\newcommand{\rewardExpectedVal}{r}

\newcommand{\mdp}{M}
\newcommand{\mdpStateSet}{S}
\newcommand{\mdpState}{s}
\newcommand{\mdpActionSet}{A}
\newcommand{\mdpAction}{a}
\newcommand{\mdpRewardFunction}{R}
\newcommand{\mdpCommonReward}{r}
\newcommand{\mdpTransition}{P}
\newcommand{\mdpDiscount}{\gamma}
\newcommand{\mdpInitialState}{\mdpState_I}
\newcommand{\mdpInitialDist}{\alpha_I}
\newcommand{\mdpCostFunction}{C}

\newcommand{\distribution}{\Delta}

\newcommand{\stateAbstraction}{\phi}

\newcommand{\policy}{\pi}
\newcommand{\targetStateSet}{\mdpStateSet_{targ}}
\newcommand{\initialStateSet}{\mdpStateSet_{init}}
\newcommand{\timeHorizon}{T}
\newcommand{\mdpSuccessState}{\mdpState_{\checkmark}}
\newcommand{\valueFunction}{V}
\newcommand{\terminationTime}{\tau}

\newcommand{\abstractMDP}{\tilde{\mdp}}
\newcommand{\abstractStateSet}{\tilde{\mdpStateSet}}
\newcommand{\abstractState}{\tilde{\mdpState}}
\newcommand{\abstractActionSet}{\tilde{\mdpActionSet}}
\newcommand{\abstractAction}{\tilde{\mdpAction}}
\newcommand{\abstractTransition}{\tilde{\mdpTransition}}
\newcommand{\abstractRewardFunction}{\tilde{\mdpRewardFunction}}
\newcommand{\abstractFailureState}{\abstractState_{\times}}
\newcommand{\abstractSuccessState}{\abstractState_{\checkmark}}
\newcommand{\abstractPolicy}{\mu}
\newcommand{\abstractInitialState}{\tilde{\mdpInitialState}}
\newcommand{\abstractInitialStateSet}{\abstractStateSet_{init}}
\newcommand{\abstractInitDist}{\alpha}
\newcommand{\abstractTargetStateSet}{\tilde{\targetStateSet}}

\newcommand{\boundMDP}{\bar{\mdp}}
\newcommand{\bernoulliProbBound}{\bar{\bernoulliProb}}
\newcommand{\boundMDPReward}{\bar{R}}
\newcommand{\boundMDPTransition}{\bar{\mdpTransition}}

\newcommand{\probThreshold}{P_{threshold}}
\newcommand{\failThreshold}{F_{threshold}}
\newcommand{\rewardThreshold}{R_{threshold}}

\newcommand{\hlmFailProb}{\delta}
\newcommand{\hlmPolicy}{\Tilde{\abstractPolicy}}

\newcommand{\controller}{c}
\newcommand{\controllerSet}{\mathcal{C}}
\newcommand{\controllerInitialStateSet}{\mathcal{I}}
\newcommand{\controllerFinalStateSet}{\mathcal{F}}
\newcommand{\controllerTimeHorizon}{T}
\newcommand{\numControllers}{k}

\newcommand{\successProb}{\sigma}

\newcommand{\eqRelation}{R}

\newcommand{\occupancyVar}{x}

\newcommand{\controllerInfProb}{\Bar{\successProb}}

\newcommand{\numSteps}{N}
\newcommand{\initTrainingSteps}{N_{init}}
\newcommand{\estimationRollouts}{N_{est}}
\newcommand{\trainingSteps}{N_{train}}
\newcommand{\maxTrainingSteps}{N_{max}}
\newcommand{\controllerPerformanceEstimate}{\hat{\successProb}}

\newcommand{\lbList}{\mathcal{L}}
\newcommand{\ubList}{\mathcal{U}}

\newcommand{\perfAwareOptProblem}{\Omega}

 \newcommand\titlesize{\fontsize{8.1pt}{10.2pt}\selectfont}

\begin{abstract}
    We propose a framework for verifiable and compositional reinforcement learning (RL) in which a collection of RL subsystems, each of which learns to accomplish a separate subtask, are composed to achieve an overall task.
    The framework consists of a \textit{high-level} model, represented as a parametric Markov decision process (pMDP) which is used to plan and to analyze compositions of subsystems, and of the collection of \textit{low-level} subsystems themselves.
    By defining interfaces between the subsystems, the framework enables automatic decompositions of task specifications, \textit{e.g., reach a target set of states with a probability of at least 0.95}, into individual subtask specifications, \textit{i.e. achieve the subsystem's exit conditions with at least some minimum probability, given that its entry conditions are met}.
    This in turn allows for the independent training and testing of the subsystems; if they each learn a policy satisfying the appropriate subtask specification, then their composition is guaranteed to satisfy the overall task specification.
    Conversely, if the subtask specifications cannot all be satisfied by the learned policies, we present a method, formulated as the problem of finding an optimal set of parameters in the pMDP, to automatically update the subtask specifications to account for the observed shortcomings.
    The result is an iterative procedure for defining subtask specifications, and for training the subsystems to meet them.
    As an additional benefit, this procedure allows for particularly challenging or important components of an overall task to be identified automatically, and focused on, during training.
    Experimental results demonstrate the presented framework's novel capabilities in both discrete and continuous RL settings.
    A collection of RL subsystems are trained, using proximal policy optimization algorithms, to navigate different portions of a labyrinth environment.
    A cross-labyrinth task specification is then decomposed into subtask specifications.
    Challenging portions of the labyrinth are automatically avoided if their corresponding subsystems cannot learn satisfactory policies within allowed training budgets. 
    Unnecessary subsystems are not trained at all. 
    The result is a compositional RL system that efficiently learns to satisfy task specifications.
    
\end{abstract}
\section{Introduction}
\label{sec:intro}

\begin{figure*}[t]
    \centering
    \input{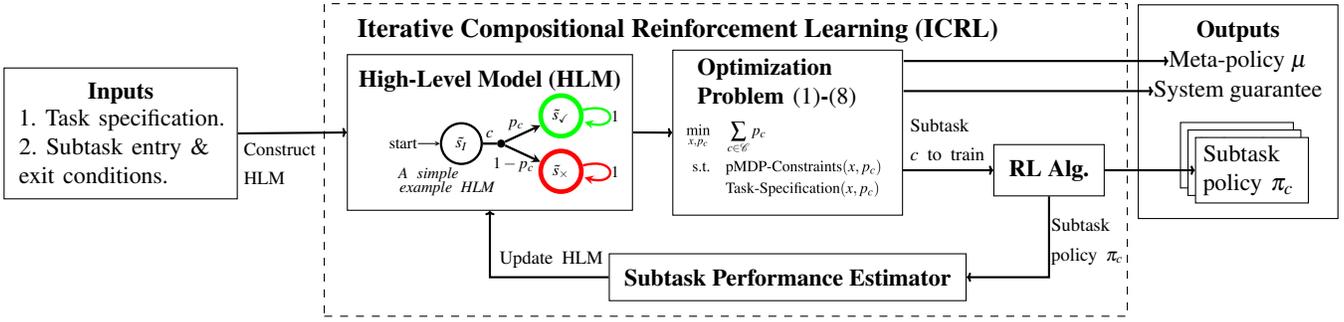}
    \caption{
    An illustration of the proposed framework.
    The task specification and the subtask entry and exit conditions are used to build the \textit{high-level model} (HLM) of the compositional RL system.
    We use the HLM to formulate an optimization problem whose outputs yield a \textit{meta-policy}, the probability of overall task success, and separate specifications for each subtask.
    The subtask specifications are used to select the next subsystem to train using the RL algorithm of choice.
    Estimates of the resulting subsystem policies are then used to update the HLM.
    This iterative process repeats until either the composite system satisfies the task specification, or a user-defined training budget has been exhausted.
    }
    \label{fig:approach_illustration}
\end{figure*} 

Reinforcement learning (RL) algorithms offer tremendous capabilities in systems that work with unknown environments.
However, there remain significant barriers to their deployment in safety-critical engineering applications.
Autonomous vehicles, manufacturing robotics, and power systems management are examples of complex application domains that require strict adherence of the system's behavior to stakeholder requirements.
However, the verification of RL systems is difficult.
This is particularly true of monolithic end-to-end RL approaches; many model-free RL algorithms, for instance, only output the learned policy and its estimated value function, rendering them opaque for verification purposes.
The difficulty of verification is compounded in engineering application domains, which often require large observation and action spaces, and complicated reward functions.

How do we build complex engineering systems we can trust?
Engineering design principles have long prescribed system modularity as a means to reduce the complexity of individual subsystems \cite{haberfellner2019systems, nuseibeh2000requirements}.
By creating well-defined interfaces between subsystems, system-level requirements may be decomposed into component-level ones. 
Conversely, each component may be developed and tested independently, and the satisfaction of component-level requirements may then be used to place assurances on the behavior of the system as a whole.
Building RL systems that incorporate such engineering practices and guarantees is a crucial step toward their widespread deployment.

Toward this end, we develop a framework for verifiable and compositional reinforcement learning.
The framework comprises two levels of abstraction.
The \textit{high level} is used to plan \textit{meta-policies} and to verify their adherence to task specifications, \textit{e.g., reach a particular goal state with a probability of at least 0.9}.
Meta-policies dictate sequences of \textit{subsystems} to execute, each of which is designed to accomplish a specific \textit{subtask}, \textit{i.e. achieve a particular exit condition, given the subsystem is executed from one of its entry conditions}.
We assume a collection of \textit{partially instantiated} subsystems to be given a priori; their entry and exit conditions are known, but the policies they implement are not.
These entry and exit conditions might be defined by pre-existing engineering capabilities, explicitly by a task designer, or by entities within the environment.
At the \textit{low level} of the framework, each subsystem employs RL algorithms to learn policies accomplishing its subtask. 
Figure \ref{fig:approach_illustration} illustrates the major components of the proposed framework.

We model the high level of the framework using a parametric Markov decision process (pMDP)~\cite{cubuktepe2018synthesis,junges2020parameter}.
Each action in the pMDP represents an individual RL subsystem, and the parametric transition probabilities in the pMDP thus represent the likelihoods of outcomes that could occur when the subsystem is executed.
Using sampling-based estimates of subsystem policies, we assign values to the model parameters and use existing MDP techniques for the planning and verification of meta-policies \cite{puterman2014markov, baier2008principles}.
Beyond this capability, the framework offers the following novel features.

\textbf{1. Automatic decomposition of task specifications.}
We formulate, as the problem of finding an optimal set of parameters in the pMDP, a method to automatically decompose the task specification into subtask specifications, allowing for independent learning and verification of the subsystems.

\textbf{2. Learning to satisfy subtask specifications.}
Any RL method can be used to learn the subsystem policies, so long as the learned policies satisfy the relevant subtask specification.
We present a subsystem reward function definition, in terms of the exit conditions of the subsystem, that motivates the learning of policies satisfying the subtask specification.
Furthermore, these subtask specifications provide an \textit{interface} between the subsystems, allowing for the analysis of their compositions. In particular, we guarantee that if each of the learned subsystem policies satisfies its subtask specifications, a composition of them exists satisfying the specifications on the overall task.

\textbf{3. Iterative specification refinement.}
However, if some of the subtask specifications cannot be satisfied by the corresponding learned policies, sampling-based estimates of their behavior are used to update the high-level model.
We present a method to use this information to refine the subtask specifications, in order to better reflect what might realistically be achieved by the subsystems.
This automatic refinement naturally leads to a compositional RL algorithm that iteratively computes subtask specifications, and then trains the corresponding subsystems to achieve them.

\textbf{4. System modularity: prediction and verification in task transfer.}
By providing an interface between the subtasks, the presented framework allows for previously learned subtask policies to be re-used as components of new high-level models, designed to solve different tasks.
Furthermore, the subtask specifications themselves may be re-used to perform verification within these new models, without the need for further training.

Experimental results exemplify these novel capabilities in both discrete and continuous versions of a labyrinth navigation task. 
We use proximal policy optimization algorithms \cite{schulman2017proximal} to train individual subsystems to navigate portions of the environment, which are then composed to complete a cross-labyrinth navigation task.
Through the aforementioned compositional RL algorithm, the task specification is decomposed and challenging portions of the labyrinth are avoided if their corresponding subsystems cannot satisfy their specifications.
\section{The Compositional RL Framework}
\label{sec:problem_statement}

\begin{figure*}[t!]
    \centering
    \begin{subfigure}[t]{0.3\textwidth}
        \centering \newcommand{\pathwidth}{2mm}
\newcommand{\circleradius}{0.5}
\newcommand{\circlecolor}{blue!30!white}
\resizebox{0.78\textwidth}{!}{
\begin{tikzpicture}[scale=1.0]
\definecolor{color0}{rgb}{0.12156862745098,0.466666666666667,0.705882352941177}
\definecolor{color1}{rgb}{0.682352941176471,0.780392156862745,0.909803921568627}
\definecolor{color2}{rgb}{1,0.498039215686275,0.0549019607843137}
\definecolor{color3}{rgb}{1,0.733333333333333,0.470588235294118}
\definecolor{color4}{rgb}{0.172549019607843,0.627450980392157,0.172549019607843}
\definecolor{color5}{rgb}{0.596078431372549,0.874509803921569,0.541176470588235}
\definecolor{color6}{rgb}{0.83921568627451,0.152941176470588,0.156862745098039}
\definecolor{color7}{rgb}{1,0.596078431372549,0.588235294117647}
\definecolor{color8}{rgb}{0.580392156862745,0.403921568627451,0.741176470588235}
\definecolor{color9}{rgb}{0.772549019607843,0.690196078431373,0.835294117647059}
\definecolor{color10}{rgb}{0.549019607843137,0.337254901960784,0.294117647058824}
\definecolor{color11}{rgb}{0.768627450980392,0.611764705882353,0.580392156862745}
\filldraw[fill=black!2!white] (0,0) rectangle (20,20);
\fill[fill=black!50!white] (0,0) rectangle (20,1);
\fill[fill=black!50!white] (0,0) rectangle (1,20);
\fill[fill=black!50!white] (0,19) rectangle (20,20);
\fill[fill=black!50!white] (19,0) rectangle (20,20);
\fill[fill=black!50!white] (0,14) rectangle (3,15);
\fill[fill=black!50!white] (4,14) rectangle (10,15);
\fill[fill=black!50!white] (5,14) rectangle (6,17);
\fill[fill=black!50!white] (5,18) rectangle (6,20);
\fill[fill=black!50!white] (11,14) rectangle (14,15);
\fill[fill=black!50!white] (15,14) rectangle (20,15);
\fill[fill=black!50!white] (8,4) rectangle (9,15);
\fill[fill=black!50!white] (0,9) rectangle (5,10);
\fill[fill=black!50!white] (6,9) rectangle (9,10);
\fill[fill=black!50!white] (0,4) rectangle (3,5);
\filldraw[fill=orange] (2,12) rectangle (5,13);
\filldraw[fill=orange] (6,11) rectangle (8,12);
\filldraw[fill=orange] (6,5) rectangle (8,6);
\filldraw[fill=orange] (3,7) rectangle (5,8);
\fill[fill=black!50!white] (4,4) rectangle (16,5);
\fill[fill=black!50!white] (13,4) rectangle (14,15);
\fill[fill=black!50!white] (17,4) rectangle (20,5);
\node at (1.9, 18.1) [fill=\circlecolor] {\fontsize{50}{60}\selectfont\(\abstractInitialState\)};
\node at (3.0,2.0) [fill=green, rectangle] {\fontsize{50}{60}\selectfont\(\controllerFinalStateSet_{targ}\)};
\filldraw[fill=\circlecolor] (3.5, 14.5) circle (\circleradius);
\filldraw[fill=\circlecolor] (5.5, 17.5) circle (\circleradius);
\filldraw[fill=\circlecolor] (10.5, 14.5) circle (\circleradius);
\filldraw[fill=\circlecolor] (14.5, 14.5) circle (\circleradius);
\filldraw[fill=\circlecolor] (5.5, 9.5) circle (\circleradius);
\filldraw[fill=\circlecolor] (3.5, 4.5) circle (\circleradius);
\filldraw[fill=\circlecolor] (12.5, 5.5) circle (\circleradius);
\filldraw[fill=\circlecolor] (16.5, 4.5) circle (\circleradius);
\filldraw[fill=\circlecolor] (10.5, 2.0) circle (\circleradius);

\draw[->, color=color0, line width=\pathwidth] (2.0, 17.5) -- (2.0, 15.5);
\draw[color=color0, line width=\pathwidth] (2.0, 15.5) -- (3.5, 15.5);
\draw[->, color=color0, line width=\pathwidth] (3.5, 15.5) -- (3.5, 15.0);
\node at (2.7, 16.5) [color=color0] {\fontsize{40}{60}\selectfont\(\controller_0\)};
\draw[->, color=color1, line width=\pathwidth] (3.0, 18.5) -- (4.5, 18.5);
\draw[color=color1, line width=\pathwidth] (4.5, 18.5) -- (4.5, 17.5);
\draw[->, color=color1, line width=\pathwidth] (4.5, 17.5) -- (5.0, 17.5);
\node at (3.7, 17.5) [color=color1] {\fontsize{40}{60}\selectfont\(\controller_1\)};
\draw[color=color2, line width=\pathwidth] (6.0, 17.5) -- (6.5, 17.5);
\draw[->, color=color2, line width=\pathwidth] (6.5, 17.5) -- (6.5, 15.5);
\draw[color=color2, line width=\pathwidth] (6.5, 15.5) -- (10.5, 15.5);
\draw[->, color=color2, line width=\pathwidth] (10.5, 15.5) -- (10.5, 15.0);
\node at (8.5, 16.2) [color=color2] {\fontsize{40}{60}\selectfont\(\controller_2\)};
\draw[->, color=color3, line width=\pathwidth] (6.0, 17.5) -- (14.5, 17.5);
\draw[->, color=color3, line width=\pathwidth] (14.5, 17.5) -- (14.5, 15.0);
\node at (15.3, 16.5) [color=color3] {\fontsize{40}{60}\selectfont\(\controller_3\)};
\draw[color=color4, line width=\pathwidth] (3.5, 14.0) -- (3.5, 13.5);
\draw[->, color=color4, line width=\pathwidth] (3.5, 13.5) -- (5.5, 13.5);
\draw[->, color=color4, line width=\pathwidth] (5.5, 13.5) -- (5.5, 10.0);
\node at (4.7, 11.0) [color=color4] {\fontsize{40}{60}\selectfont\(\controller_4\)};
\draw[->, color=color5, line width=\pathwidth] (5.5, 9.0) -- (5.5, 5.5);
\draw[color=color5, line width=\pathwidth] (5.5, 5.5) -- (3.5, 5.5);
\draw[->, color=color5, line width=\pathwidth] (3.5, 5.5) -- (3.5, 5.0);
\node at (6.3, 8.0) [color=color5] {\fontsize{40}{60}\selectfont\(\controller_5\)};
\draw[color=color6, line width=\pathwidth] (10.5, 14.0) -- (10.5, 13.5);
\draw[color=color6, line width=\pathwidth] (10.5, 13.5) -- (9.5, 13.5);
\draw[->, color=color6, line width=\pathwidth] (9.5, 13.5) -- (9.5, 9.5);
\draw[->, color=color6, line width=\pathwidth] (9.5, 9.5) -- (9.5, 5.5);
\draw[->, color=color6, line width=\pathwidth] (9.5, 5.5) -- (12.0, 5.5);
\node at (10.3, 12.0) [color=color6] {\fontsize{40}{60}\selectfont\(\controller_6\)};
\draw[->, color=color7, line width=\pathwidth] (12.5, 6.0) -- (12.5, 9.5);
\draw[->, color=color7, line width=\pathwidth] (12.5, 9.5) -- (12.5, 13.5);
\draw[color=color7, line width=\pathwidth] (12.5, 13.5) -- (10.5, 13.5);
\draw[color=color7, line width=\pathwidth] (10.5, 13.5) -- (10.5, 14.0);
\node at (11.7, 8.0) [color7] {\fontsize{40}{60}\selectfont\(\controller_7\)};
\draw[->, color=color8, line width=\pathwidth] (14.5, 14.0) -- (14.5, 9.5);
\draw[->, color=color8, line width=\pathwidth] (14.5, 9.5) -- (16.5, 9.5);
\draw[->, color=color8, line width=\pathwidth] (16.5, 9.5) -- (16.5, 5.0);
\node at (17.0, 10.0) [color=color8] {\fontsize{40}{60}\selectfont\(\controller_8\)};
\draw[->, color=color9, line width=\pathwidth] (3.5, 4.0) -- (3.5, 3.0);
\node at (4.3, 3.5) [color=color9] {\fontsize{40}{60}\selectfont\(\controller_9\)};
\draw[->, color=color10, line width=\pathwidth] (16.5, 4.0) -- (16.5, 2.0);
\draw[->, color=color10, line width=\pathwidth] (16.5, 2.0) -- (11.0, 2.0);
\node at (13.5, 2.7) [color=color10] {\fontsize{40}{60}\selectfont\(\controller_{10}\)};
\draw[->, color=color11, line width=\pathwidth] (10.0, 2.0) -- (5.0, 2.0);
\node at (7.5, 2.7) [color=color11] {\fontsize{40}{60}\selectfont\(\controller_{11}\)};
\end{tikzpicture}%
}%
        \caption{The labyrinth task environment.}
        \vspace*{0.5cm}
        \label{fig:labyrinth_gridworld}
    \end{subfigure}%
    ~
    \hspace*{18mm}
    \begin{subfigure}[t]{0.45\textwidth}
        \centering 

\tikzstyle{branch}=[fill,shape=circle,minimum size=5pt,inner sep=0pt]
\def\horizontaldistance{2.5cm}
\def\verticaldistance{2.0cm}
\def\nodedistance{2.5cm}
\def\branchdist{0.8cm}

\def\stateOutlineThickness{0.5mm}
\def\edgeThickness{0.5mm}

\tikzset{auto, ->, >=stealth, node distance=\nodedistance, node/.style={scale=0.8, minimum size=0pt, inner sep=0pt}}

\resizebox{1.0\textwidth}{!}{
\begin{tikzpicture}[scale=1.0]
    \node[state, line width=\stateOutlineThickness] (s_init) {\(\abstractInitialState\)};
    \path (s_init)+(\horizontaldistance, 0.0) node (s0) [state, line width=\stateOutlineThickness] {\(\abstractState_0\)};
    \path (s_init)+(0.0, -\verticaldistance) node (s1) [state, line width=\stateOutlineThickness] {\(\abstractState_1\)};
    \path (s1)+(0.0, -\verticaldistance) node (s4) [state, line width=\stateOutlineThickness] {\(\abstractState_4\)};
    \path (s4)+(\horizontaldistance, 0.0) node (s5) [state, line width=\stateOutlineThickness] {\(\abstractState_5\)};
    
    \path (s0)+(1.8cm, -1.0cm) node (s2) [state, line width=\stateOutlineThickness] {\(\abstractState_2\)};
    \path (s2)+(0.0, -\verticaldistance) node (s6) [state, line width=\stateOutlineThickness] {\(\abstractState_6\)};
    
    \path (s0)+(5.0cm, 0.0) node (s3) [state, line width=\stateOutlineThickness] {\(\abstractState_3\)};
    \path (s3)+(0.0, -\verticaldistance) node (s7) [state, line width=\stateOutlineThickness] {\(\abstractState_7\)};
    \path (s7)+(0.0, -\verticaldistance) node (s8) [state, line width=\stateOutlineThickness] {\(\abstractState_8\)};
    
    \path (s1)+(2.2cm, 0.0) node (s_fail) [state, draw=red, line width=1.0mm] {\(\abstractFailureState\)};
    \path (s_fail)+(3.7cm, 0.0) node (s_fail2) [state, draw=red, line width=1.0mm] {\(\abstractFailureState\)};
    \path (s5)+(3.0cm, 0.0) node (s_goal) [state, draw=green, line width=1.0mm] {\(\abstractSuccessState\)};
    
    \path (s_init)+(\branchdist, 0.0) node [branch] (c1) {};
    \path [draw, line width=\edgeThickness, -] (s_init.east) -- node [above] {\(\controller_1\)} (c1);
    \path [draw, ->, line width=\edgeThickness] (c1) -- node [above] {\(\bernoulliProb_{\controller_1}\)} (s0);
    \path [draw=red, ->, line width=\edgeThickness] (c1) -- node [above] {} (s_fail);
    \path (s_init)+(0.0, -\branchdist) node [branch] (c0) {};
    \path [draw, -, line width=\edgeThickness] (s_init.south) -- node [left] {\(\controller_0\)} (c0);
    \path [draw, ->, line width=\edgeThickness] (c0) -- node [left] {\(\bernoulliProb_{\controller_0}\)} (s1);
    \path [draw=red, ->, line width=\edgeThickness] (c0) -- node [above] {} (s_fail);
    \path (s0)+(0.87cm, -0.33cm) node [branch] (c2) {};
    \path [draw, -, line width=\edgeThickness] (s0.east) -- node [below] {\(\controller_2\)} (c2);
    \path [draw, ->, line width=\edgeThickness] (c2) -- node [below] {\(\bernoulliProb_{\controller_2}\)} (s2);
    \path [draw=red, ->, line width=\edgeThickness] (c2) -- node [above] {} (s_fail);
    \path (s0)+(2.2cm, 0.0cm) node [branch] (c3) {};
    \path [draw, -, line width=\edgeThickness] (s0.east) -- node [above] {\(\controller_3\)} (c3);
    \path [draw, ->, line width=\edgeThickness] (c3) -- node [above] {\(\bernoulliProb_{\controller_3}\)} (s3);
    \path (c3)+(-0.5, -0.5) node (c3_fail) {};
    \path [draw=red, ->, line width=\edgeThickness] (c3) -- node [above] {} (s_fail2);
    \path (s2)+(-0.4cm, -\branchdist) node [branch] (c6) {};
    \path [draw, -, line width=\edgeThickness] (s2.south) -- node [left] {\(\controller_6\)} (c6);
    \path [draw, ->, line width=\edgeThickness] (c6) -- node [left] {\(\bernoulliProb_{\controller_6}\)} (s6.north);
    \path [draw=red, ->, line width=\edgeThickness] (c6) -- node [above] {} (s_fail);
    \path (s6)+(0.4cm, 0.8cm) node [branch] (c7) {};
    \path [draw, -, line width=\edgeThickness] (s6.north) -- node [right] {\(\controller_7\)} (c7);
    \path [draw, ->, line width=\edgeThickness] (c7) -- node [right] {\(\bernoulliProb_{\controller_7}\)} (s2.south);
    \path (c7)+(-0.6, 0.05) node (c7_fail) {};
    \path [draw=red, ->, line width=\edgeThickness] (c7) -- node [above] {} (s_fail2);
    \path (s1)+(0.0, -\branchdist) node [branch] (c4) {};
    \path [draw, -, line width=\edgeThickness] (s1.south) -- node [left] {\(\controller_4\)} (c4);
    \path [draw, ->, line width=\edgeThickness] (c4) -- node [left] {\(\bernoulliProb_{\controller_4}\)} (s4);
    \path [draw=red, ->, line width=\edgeThickness] (c4) -- node [above] {} (s_fail);
    \path (s4)+(\branchdist, 0.0) node [branch] (c5) {};
    \path [draw, -, line width=\edgeThickness] (s4.east) -- node [below] {\(\controller_5\)} (c5);
    \path [draw, ->, line width=\edgeThickness] (c5) -- node [below] {\(\bernoulliProb_{\controller_5}\)} (s5);
    \path [draw=red, ->, line width=\edgeThickness] (c5) -- node [above] {} (s_fail);
    \path (s5)+(\branchdist, 0.0) node [branch] (c9) {};
    \path [draw, -, line width=\edgeThickness] (s5.east) -- node [below] {\(\controller_9\)} (c9);
    \path [draw, ->, line width=\edgeThickness] (c9) -- node [below] {\(\bernoulliProb_{\controller_9}\)} (s_goal);
    \path [draw=red, ->, line width=\edgeThickness] (c9) -- node [above] {} (s_fail);
    \path (s3)+(0.0, -\branchdist) node [branch] (c8) {};
    \path [draw, -, line width=\edgeThickness] (s3.south) -- node [right] {\(\controller_8\)} (c8);
    \path [draw, ->, line width=\edgeThickness] (c8) -- node [right] {\(\bernoulliProb_{\controller_8}\)} (s7);
    \path (c8)+(-1.0, -0.3) node (c8_fail) {};
    \path [draw=red, ->, line width=\edgeThickness] (c8) -- node [above] {} (s_fail2);
    
    \path (s7)+(0.0, -\branchdist) node [branch] (c10) {};
    \path [draw, -, line width=\edgeThickness] (s7.south) -- node [right] {\(\controller_{10}\)} (c10);
    \path [draw, ->, line width=\edgeThickness] (c10) -- node [right] {\(\bernoulliProb_{\controller_{10}}\)} (s8);
    \path (c10)+(-1.0, 0.2) node (c10_fail) {};
    \path [draw=red, ->, line width=\edgeThickness] (c10) -- node [above] {} (s_fail2);
    \path (s8)+(-\branchdist, 0.0) node [branch] (c11) {};
    \path [draw, -, line width=\edgeThickness] (s8.west) -- node [below] {\(\controller_{11}\)} (c11);
    \path [draw, ->, line width=\edgeThickness] (c11) -- node [below] {\(\bernoulliProb_{\controller_{11}}\)} (s_goal);
    \path (c11)+(-0.8, 0.5) node (c11_fail) {};
    \path [draw=red, ->, line width=\edgeThickness] (c11) -- node [above] {} (s_fail2);
    
    \path (s_fail) edge [loop above, draw=red, line width=\edgeThickness] node {1} (s_fail);
    \path (s_fail2) edge [loop above, draw=red, line width=\edgeThickness] node {1} (s_fail2);
    \path (s_goal) edge [loop above, draw=green, line width=\edgeThickness] node {1} (s_goal);
\end{tikzpicture}
}
        \vspace*{-0.4cm}
        \caption{The HLM corresponding to the labyrinth example.}
        \vspace*{0.5cm}
        \label{fig:labyrinth_HLM}
    \end{subfigure}
    \caption{An example labyrinth navigation task. Figure (a) illustrates the environment, as well as an example collection of subsystems, represented by the colored paths. Entry and exit conditions for the various subsystems are shown as blue circles. Figure (b) illustrates the corresponding HLM. Each subsystem \(\controller\) causes a transition to its successor state with probability \(\bernoulliProb_{\controller}\). Otherwise, the HLM transitions to the failure state \(\abstractFailureState\) with probability \(1 - \bernoulliProb_{\controller}\), visualized by the red transitions.} 
\end{figure*}
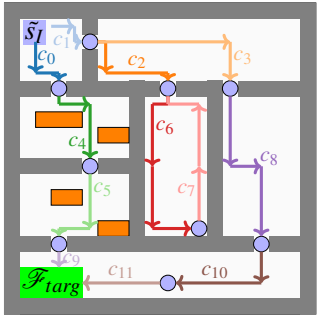
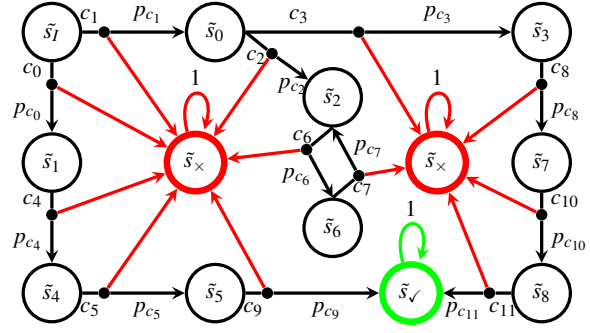

To provide intuitive examples of the notions of tasks, subtasks, systems, and subsystems, we consider the example labyrinth environment shown in Figure \ref{fig:labyrinth_gridworld}. 
The \textit{system} executes its constituent \textit{subsystems} in this environment to complete an overall task.
The \textit{task} is to safely navigate from the labyrinth's initial state in the top left corner to the goal state in the bottom left corner. Satisfaction of the \textit{task specification} requires that the system successfully completes the task with a probability of at least \(0.95\).
As an added difficulty, lava exists within some of the rooms, represented in the figure by the orange rectangles. 
If the lava is touched, the task is automatically failed. 
This task is naturally decomposed into separate \textit{subtasks}, each of which navigates an individual room, and is executed by a separate subsystem.

\subsubsection{Preliminaries.}
We model the task environment as a Markov decision process (MDP), which is defined by a tuple \(\mdp = (\mdpStateSet, \mdpActionSet, \mdpTransition)\). 
Here, \(\mdpStateSet\) is a set of states, \(\mdpActionSet\) is a set of actions, and \(\mdpTransition : \mdpStateSet \times \mdpActionSet \times \mdpStateSet \to [0,1]\) is a transition probability function.
A stationary policy \(\policy\) within the MDP is a function \(\policy : \mdpStateSet \times \mdpActionSet \to [0,1]\) such that \(\sum_{\mdpAction \in \mdpActionSet} \policy(\mdpState, \mdpAction) = 1\) for every \(\mdpState \in \mdpStateSet\).
Intuitively, \(\policy(\mdpState, \mdpAction)\) assigns the probability of taking action \(\mdpAction\) from state \(\mdpState\) under policy \(\policy\).
Given an MDP \(\mdp\), a policy \(\policy\), and a target set of states \(\targetStateSet \subseteq \mdpStateSet\), we define \(\mathbb{P}^{\mdpState}_{\mdp, \policy}(\Diamond \targetStateSet)\) to be the probability of eventually reaching some state \(\mdpState' \in \targetStateSet\), beginning from the initial state \(\mdpState\), under policy \(\policy\).
Similarly, \(\mathbb{P}^{\mdpState}_{\mdp, \policy}(\Diamond_{\leq \timeHorizon} \targetStateSet)\) denotes the probability of reaching the target set from state \(\mdpState\) within some finite time horizon \(\timeHorizon\).

The framework we present is agnostic to the implementation details of the RL algorithms that interact with the low-level environment.
As such, \(\mdpStateSet\) and \(\mdpActionSet\) can either be uncountably infinite subsets of Euclidean space, or they can be countable sets indexing the states and actions. 
Our experiments examine both cases.
For notational simplicity, we present the framework for countable sets \(\mdpStateSet\) and \(\mdpActionSet\).

\subsubsection{RL Subsystems and Subtasks.}
We define each RL subsystem \(\controller\) acting within the environment  by the tuple \(\controller = (\controllerInitialStateSet_{\controller}, \controllerFinalStateSet_{\controller}, \controllerTimeHorizon_{\controller}, \policy_{\controller})\).
Here, \(\controllerInitialStateSet_{\controller} \subseteq \mdpStateSet\) is a set defining the subsystem's \textit{entry conditions}, \(\controllerFinalStateSet_{\controller} \subseteq \mdpStateSet\) is a set representing the subsystem's \textit{exit conditions}, and \(\timeHorizon_{\controller} \in \mathbb{N}\) is the subsystem's allowed \textit{time horizon}.
The \textit{subtask} associated with each subsystem, is to navigate from any entry condition \(\mdpState \in \controllerInitialStateSet_{\controller}\) to any exit condition \(\mdpState' \in \controllerFinalStateSet_{\controller}\) within the subsystem's time horizon \(\timeHorizon_{\controller}\). 
The time horizon is included to ensure that the compositional system will complete its task in finite time.
We assume that each subsystem may only be \textit{executed}, or begun, from an entry condition \(\mdpState \in \controllerInitialStateSet_{\controller}\) and that its execution ends either when it achieves an exit condition \(\mdpState \in \controllerFinalStateSet_{\controller}\), or when it runs out of time.
Finally, \(\policy_{\controller} : \mdpStateSet \times \mdpActionSet \to [0,1]\) is the policy that the component implements to complete this objective.

For notational convenience, we define \(\successProb_{\policy_{\controller}}^{\controller}(\mdpState) \defeq \mathbb{P}^{\mdpState}_{\mdp, \policy_{\controller}}(\Diamond_{\leq \timeHorizon_{\controller}}\controllerFinalStateSet_{\controller})\).
A \textit{subtask specification}, is then defined as the requirement that \(\successProb_{\policy_{\controller}}^{\controller}(s) \geq \bernoulliProb_{\controller}\) for every entry condition \(\mdpState \in \controllerInitialStateSet_{\controller}\) of the subsystem. Here, \(\bernoulliProb_{\controller} \in [0,1]\) is a value representing the minimum allowable probability of the subtask success. 
We note that such reachability-based task specifications are very expressive. Temporal logic specifications can be expressed as reachability specifications in a so-called product MDP \cite{baier2008principles, hahn2019omega}.

We say a subsystem \(\controller\) is \textit{partially instantiated} when \(\controllerInitialStateSet_{\controller}\), \(\controllerFinalStateSet_{\controller}\), and \(\timeHorizon_{\controller}\) are defined, but its policy \(\policy_{\controller}\) is not.
We define a collection \(\controllerSet = \{\controller_1, \controller_2, ..., \controller_{\numControllers}\}\) of subsystems to be \textit{composable}, if and only if for every \(i,j \in \{1,2,\ldots, \numControllers\}\), either \(\controllerFinalStateSet_{\controller_i} \subseteq \controllerInitialStateSet_{\controller_j}\) or \(\controllerFinalStateSet_{\controller_i} \cap \controllerInitialStateSet_{\controller_j} = \emptyset\).
In words, subsystems are composable when the set of exit conditions of each subsystem is a subset of all the sets of entry conditions that it intersects.
This ensures that regardless of the specific exit condition \(\mdpState \in \controllerFinalStateSet_{\controller}\) in which subsystem \(\controller\) terminates, \(\mdpState\) will be a valid entry condition for the \textit{same} collection of other subsystems.

\subsubsection{Compositions of RL Subsystems.}
Compositions of subsystems are specified by \textit{meta-policies} \(\abstractPolicy : \mdpStateSet \times \controllerSet \to [0,1]\), which assign probability values to the execution of different subsystems, given the current environment state \(\mdpState \in \mdpStateSet\).
So, execution of the composite system occurs as follows.
From a given state \(\mdpState\), the meta-policy is used to select a subsystem \(\controller\) to execute. 
The subsystem's policy \(\policy_{\controller}\) is then followed until it either reaches an exit condition \(\mdpState' \in \controllerFinalStateSet_{\controller}\), or it reaches the end of its time horizon \(\timeHorizon_{\controller}\). 
If the former is true, the meta-policy selects the next subsystem to execute from \(\mdpState'\), and the process repeats.
Conversely, if the latter is true, the subsystem has failed to complete its subtask in time, and the execution of the meta-policy stops.
In the labyrinth example, the meta-policy selects which rooms to pass through, while the subsystems policies navigate the individual rooms.

The \textit{task} of the composite system is, beginning from an initial state \(\mdpInitialState\), to eventually reach a particular target exit condition \(\controllerFinalStateSet_{targ} \subseteq \mdpStateSet\). 
We assume that  \(\controllerFinalStateSet_{targ}\) is equivalent to \(\controllerFinalStateSet_{\controller}\) for at least one of the subsystems. 
That is, there is some subsystem \(\controller \in \controllerSet\) such that \(\controllerFinalStateSet_{targ} = \controllerFinalStateSet_{\controller}\).
Furthermore, to simplify theoretical analysis, we assume that for every \(\controller \in \controllerSet\), either \(\controllerFinalStateSet_{\controller} = \controllerFinalStateSet_{targ}\) or \(\controllerFinalStateSet_{\controller} \cap \controllerFinalStateSet_{targ} = \emptyset\).
This assumption removes ambiguity as to whether or not completion of a given subtask results in the immediate completion of the system's task.
Finally, we assume that at least one subsystem \(\controller\) can be executed from the initial state \(\mdpInitialState\), i.e. there exists a subsystem \(\controller \in \controllerSet\) such that \(\mdpInitialState \in \controllerInitialStateSet_{\controller}\).
We say that the execution of a meta-policy reaches the target set \(\controllerFinalStateSet_{targ}\), when one of the subsystems \(\controller\) with \(\controllerFinalStateSet_{\controller} = \controllerFinalStateSet_{targ}\) is executed, and successfully completes its subtask. With a slight abuse of notation, we denote the probability of eventually reaching the target set under meta-policy \(\abstractPolicy\) by \(\mathbb{P}_{\mdp, \abstractPolicy}^{\mdpInitialState}(\Diamond \controllerFinalStateSet_{targ})\).

A \textit{task specification} places a requirement on the probability of the compositional RL system reaching \(\controllerFinalStateSet_{targ}\). That is, for some allowable failure probability \(\hlmFailProb \in [0,1]\), the task specification is satisfied if \(\mathbb{P}_{\mdp, \abstractPolicy}^{\mdpInitialState}(\Diamond \controllerFinalStateSet_{targ}) \geq 1 - \hlmFailProb\).
With these definitions in place, we now deliver our problem statement.

\textbf{Problem Statement. }\textit{Given an allowable failure probability \(\hlmFailProb \in [0,1]\), an initial state \(\mdpInitialState\), a target set \(\controllerFinalStateSet_{targ}\), and a partially instantiated collection \(\controllerSet\) of composable subsystems, learn policies \(\policy_{\controller}\) for each subsystem \(\controller \in \controllerSet\) and compute a meta-policy \(\abstractPolicy\) such that \(\mathbb{P}^{\mdpInitialState}_{\mdp, \abstractPolicy}(\Diamond \controllerFinalStateSet_{targ}) \geq 1 - \hlmFailProb\).}
\section{The High-Level Decision-Making Model}\label{sec:hlm}
We now introduce the high-level model (HLM) of the compositional RL framework, which is used to compute meta-policies, and to decompose task specifications into subtask specifications to be satisfied by the individual subsystems.

\subsubsection{Defining the High-Level Model (HLM).}
To construct the HLM, we use a given collection \(\controllerSet = \{\controller_1, \controller_2, \ldots, \controller_{\numControllers}\}\) of partially instantiated subsystems, an initial state \(\mdpInitialState\), and a target set \(\controllerFinalStateSet_{targ}\).
We begin by defining a state abstraction, which groups together environment states in order to define the state space of the HLM.
To do so, we define the equivalence relation \(\eqRelation \subseteq \mdpStateSet \times \mdpStateSet\) such that \((\mdpState, \mdpState') \in \eqRelation\) if and only if the following two conditions hold.
\begin{align*}
	\begin{array}{l}
		 \text{1. For every \(\controller \in \controllerSet, \mdpState \in \controllerInitialStateSet_{\controller}\) if and only if \(\mdpState' \in \controllerInitialStateSet_{\controller}\), and}, \\
		 \text{2. \(\mdpState \in \controllerFinalStateSet_{targ}\) if and only if \(\mdpState' \in \controllerFinalStateSet_{targ}\).}
	\end{array}
\end{align*}%
The equivalence class of any state \(\mdpState \in \mdpStateSet\) under equivalence relation \(\eqRelation\) is given by \([\mdpState]_{\eqRelation} = \{\mdpState' \in \mdpStateSet | (\mdpState, \mdpState') \in \eqRelation\}\). 
The quotient set of \(\mdpStateSet\) by \(\eqRelation\) is defined as the set of all equivalence classes \(\mdpStateSet /_{\eqRelation} = \{[\mdpState]_{\eqRelation} | \mdpState\in \mdpStateSet\}\).
Intuitively, this equivalence relation groups together all the states in the target set, and it also groups together states that are entry conditions to the same subset of subsystems.

We may now define the HLM corresponding to the collection \(\controllerSet\) by the parametric MDP \(\abstractMDP = (\abstractStateSet, \abstractInitialState, \abstractSuccessState, \abstractFailureState, \controllerSet, \abstractTransition)\).
Here, the high-level states \(\abstractStateSet\) are defined to be \(\mdpStateSet/_{\eqRelation}\); states in the HLM correspond to equivalence classes of environment states.
The initial state \(\abstractInitialState\) of the HLM is defined as \(\abstractInitialState = [\mdpInitialState]_{\eqRelation}\), the equivalence class of the environment's initial state. 
The \textit{goal state} \(\abstractSuccessState \in \abstractStateSet\) is similarly defined as \([\mdpState]_{\eqRelation}\) such that \(\mdpState \in \controllerFinalStateSet_{targ}\). 
Recall that \(\controllerFinalStateSet_{targ} = \controllerFinalStateSet_{\controller}\) for at least one of the subsystems \(\controller \in \controllerSet\).
Finally, the \textit{failure state} \(\abstractFailureState \in \abstractStateSet\) is defined as \([\mdpState]_{\eqRelation}\) such that \(\mdpState \in \mdpStateSet \setminus [\bigcup_{\controller \in \controllerSet} \controllerInitialStateSet_{\controller}] \cup \controllerFinalStateSet_{targ}\), i.e., the equivalence class of states \textit{not} belonging to the initial states of any component, or to the target set.

As an example, Figure \ref{fig:labyrinth_HLM} illustrates the HLM corresponding to the collection of subsystems from Figure \ref{fig:labyrinth_gridworld}. The overlapping entry and exit conditions, represented by the blue circles in Figure \ref{fig:labyrinth_gridworld}, define the states of the HLM. The target set \(\controllerFinalStateSet_{targ}\) defines the HLM's goal state \(\abstractSuccessState\), and all other environment states are absorbed into the failure state \(\abstractFailureState\).

The collection of subsystems \(\controllerSet\) defines the HLM's set of actions. By definition of the equivalence relation \(\eqRelation\), for every HLM state \(\abstractState \in \abstractStateSet\) there is a well-defined subset of the subsystems \(\controllerSet(\abstractState) \subseteq \controllerSet\) that can be executed.
That is, for every environment state \(\mdpState \in \abstractState\), \(\mdpState \in \controllerInitialStateSet_{\controller}\) for all \(\controller \in \controllerSet(\abstractState)\).
We define \(\controllerSet(\abstractState)\) as the set of \textit{available subsystems} at high-level state \(\abstractState\).

Furthermore, consider any subsystem \(\controller \in \controllerSet(\abstractState)\). 
As a direct result of the definition of equivalence relation \(\eqRelation\) and of the subsystems in collection \(\controllerSet\) being composable, every state \(\mdpState\) within set \(\controllerFinalStateSet_{\controller}\) belongs to the \textit{same} equivalence class \([\mdpState]_{\eqRelation}\).
In other words, we may uniquely define the successor HLM state of any component \(\controller \in \controllerSet\) as \(succ(\controller) = [\mdpState]_{\eqRelation}\) such that \(\mdpState \in \controllerFinalStateSet_{\controller}\). We then construct the HLM transition probability function in terms of parameters \(\bernoulliProb_{\controller} \in [0,1]\) as follows.%
\begin{align*}
\abstractTransition(\abstractState, \controller, \abstractState') = \begin{cases} 
  \bernoulliProb_{\controller}, & if\; \; \controller \in \controllerSet(\abstractState), \;\; \abstractState' = succ(\controller) \\
  1 - \bernoulliProb_{\controller}, & if \; \; \controller\in \controllerSet(\abstractState), \; \; \abstractState' = \abstractFailureState\\
  0, & \text{Otherwise}
\end{cases}\end{align*}%
The interpretation of this definition of \(\abstractTransition\) is as follows. 
After selecting component \(\controller \in \controllerSet(\abstractState)\) from HLM state \(\abstractState\), the component either succeeds in reaching an exit condition  \(\mdpState \in \controllerFinalStateSet_{\controller}\) within its time horizon \(\timeHorizon_{\controller}\) with probability \(\bernoulliProb_{\controller}\), resulting in an HLM transition to \(succ(\controller)\), or it fails to do so with probability \(1 - \bernoulliProb_{\controller}\), resulting in a transition to the HLM failure state \(\abstractFailureState\). 

The parameters \(\bernoulliProb_{\controller}\) may thus be interpreted as estimates of the probabilities that the subsystems complete their subtasks, given they are executed from one of their entry conditions.
Their values come either from empirical rollouts of learned subsystem policies \(\policy_{\controller}\), or as the solution to the aforementioned automatic decomposition of the task specification, which is discussed further below.

\subsubsection{Relating the HLM to Compositions of RL Subsystems.}
We note that while parameters \(\bernoulliProb_{\controller}\) are meant to estimate the probabilities of successful subtask completion, they cannot capture these probabilities exactly.
In reality, while parameter \(\bernoulliProb_{\controller}\) is constant, it's possible for this probability to vary, given the entry condition \(\mdpState \in \controllerInitialStateSet_{\controller}\) from which the component is executed.
However, the simplicity of the presented parametrization of \(\abstractTransition\) enables tractable solutions to planning and verification problems in \(\abstractMDP\).
Furthermore, by establishing relationships between policies in \(\abstractMDP\), and meta-policies composing RL subsystems, the HLM becomes practically useful in the analysis of composite RL systems.

Towards this idea, we note that any stationary policy \(\hlmPolicy : \abstractStateSet \times \controllerSet \to [0,1]\) acting in HLM \(\abstractMDP\) defines a unique compositional meta-policy \(\abstractPolicy : \mdpStateSet \times \controllerSet \to [0,1]\) as follows: for any environment state \(\mdpState\) and component \(\controller\), define  \(\abstractPolicy(\mdpState, \controller) \defeq \hlmPolicy([\mdpState]_{\eqRelation}, \controller)\).
So, solutions to planning problems in \(\abstractMDP\) can be used directly as meta-policies to specify compositions of the RL subsystems. 
Of particular interest, is the problem of computing an HLM policy \(\hlmPolicy\) that maximizes \(\mathbb{P}^{\abstractInitialState}_{\abstractMDP, \hlmPolicy}(\Diamond \abstractSuccessState)\), the probability of eventually reaching the goal state \(\abstractSuccessState\) from the HLM's initial state \(\abstractInitialState\).
Theorem \ref{thm:hlm_bounds_true_performance} relates this probability to the corresponding meta-policy's probability of completing its task, \(\mathbb{P}_{\mdp, \abstractPolicy}^{\mdpInitialState}(\Diamond \controllerFinalStateSet_{targ})\), in the environment.%
\begin{theorem}
\label{thm:hlm_bounds_true_performance}
Let \(\controllerSet = \{\controller_1, \controller_2, ..., \controller_{\numControllers}\}\) be a collection of composable subsystems with respect to initial state \(\mdpInitialState\) and target set \(\controllerFinalStateSet_{targ}\) within the environment MDP \(\mdp\). Define \(\abstractMDP\) to be the corresponding HLM and let \(\hlmPolicy\) be a policy in \(\abstractMDP\). If, for every subsystem \(\controller \in \controllerSet\) and for every entry condition \(\mdpState \in \controllerInitialStateSet_{\controller}\), \(\successProb_{\policy_{\controller}}^{\controller}(\mdpState) \geq \bernoulliProb_{\controller}\), then \(\mathbb{P}^{\mdpInitialState}_{\mdp, \abstractPolicy}(\Diamond \controllerFinalStateSet_{targ}) \geq \mathbb{P}^{\abstractInitialState}_{\abstractMDP, \hlmPolicy}(\Diamond \abstractSuccessState)\).
\end{theorem}%
For example, consider the labyrinth task from Figure \ref{fig:labyrinth_gridworld}, and its corresponding HLM from Figure \ref{fig:labyrinth_HLM}. Suppose the HLM's parameters \(\bernoulliProb_{\controller}\) are specified such that they lower bound the true probabilities of subtask success, i.e. the transition probabilities in Figure \ref{fig:labyrinth_HLM} lower bound the probabilities of the subsystems successfully navigating their respective rooms in Figure \ref{fig:labyrinth_gridworld}. 
By planning a policy \(\hlmPolicy\) in the HLM that, for example, reaches \(\abstractSuccessState\) with probability \(0.95\), we ensure that the corresponding composition of the  subsystems will reach \(\controllerFinalStateSet_{targ}\) in the labyrinth with a probability of \textit{at least} \(0.95\).

\subsubsection{Automatic Decomposition of Task Specifications.}
Recall that our objective is not only to compute a meta-policy \(\abstractPolicy\), but also to \textit{learn} the subsystem policies \(\policy_{\controller_1}, \policy_{\controller_2},..., \policy_{\controller_{\numControllers}}\) that this meta-policy will execute, such that the system's task specification \(\mathbb{P}^{\mdpInitialState}_{\mdp, \abstractPolicy}(\Diamond \controllerFinalStateSet_{targ}) \geq 1 - \hlmFailProb\) is satisfied.
Suppose that we choose a set of HLM parameters \(\{\bernoulliProb_{\controller_1}, \bernoulliProb_{\controller_2}, ..., \bernoulliProb_{\controller_{\controller_{\numControllers}}}\}\) such that a policy \(\hlmPolicy\) in the HLM exists with \(\mathbb{P}^{\mdpInitialState}_{\mdp, \abstractPolicy}(\Diamond \targetStateSet) \geq 1 - \hlmFailProb\).
Then, so long as each of the corresponding subsystems \(\controller\) are able to learn a policy \(\policy_{\controller}\) such that \(\successProb_{\policy_{\controller}}^{\controller}(\mdpState) \geq \bernoulliProb_{\controller}\) for every \(\mdpState \in \controllerInitialStateSet_{\controller}\), Theorem \ref{thm:hlm_bounds_true_performance} tells us that the meta-policy defined by \(\abstractPolicy(\mdpState, \controller) \defeq \hlmPolicy([\mdpState]_{\eqRelation}, \controller)\) will satisfy the task specification.

We may thus interpret the values of parameters \(\bernoulliProb_{\controller}\) as \textit{subtask specifications}.
Each subsystem must achieve one of its exit conditions \(\mdpState' \in \controllerFinalStateSet_{\controller}\) within its allowed time horizon \(\timeHorizon_{\controller}\) with a probability of at least \(\bernoulliProb_{\controller}\), given its execution began from some entry condition \(\mdpState \in \controllerInitialStateSet_{\controller}\).
With this interpretation in mind, we take the following approach to the decomposition of the task specification: find the smallest values of parameters \(\bernoulliProb_{\controller_1}, \bernoulliProb_{\controller_2}, ..., \bernoulliProb_{\controller_{\numControllers}}\) such that an HLM policy \(\hlmPolicy\) exists satisfying \(\mathbb{P}^{\abstractInitialState}_{\abstractMDP, \hlmPolicy}(\Diamond \abstractSuccessState) \geq 1 - \hlmFailProb\).
We formulate this constrained parameter optimization problem as the bilinear program given in equations (\ref{eq:hlm_opt_objective})-(\ref{eq:hlm_opt_task_sat_constraints}).
In (\ref{eq:hlm_opt_dynamics_constraints}) and (\ref{eq:hlm_opt_task_sat_constraints}), we define \(pred(\abstractState) \defeq \{(\abstractState', \controller') | \controller' \in \controllerSet(\abstractState') \; and \; \abstractState = succ(\controller')\}\).%
\begin{align}
    \min_{\occupancyVar, \bernoulliProb_{\controller}} \quad & \sum_{\controller \in \controllerSet} \bernoulliProb_{\controller}\label{eq:hlm_opt_objective} \\
    \textrm{s.t.} \quad & \sum_{\controller \in \controllerSet(\abstractState)} \occupancyVar(\abstractState, \controller) = \delta_{\abstractInitialState}(\abstractState) + \sum_{(\abstractState', \controller') \in pred(\abstractState)} \occupancyVar(\abstractState', \controller') \bernoulliProb_{\controller'}, \label{eq:hlm_opt_dynamics_constraints} \\
    & \qquad \qquad \qquad \qquad \qquad \qquad \forall \abstractState \in  \abstractStateSet \setminus \{\abstractFailureState, \abstractSuccessState\} \notag\\
    & \occupancyVar(\abstractState, \controller) \geq 0, \; \forall \abstractState \in \abstractStateSet \setminus \{\abstractFailureState, \abstractSuccessState\},\; \forall \controller \in \controllerSet(\abstractState) \\
    & 0 \leq \bernoulliProb_{\controller} \leq 1, \; \forall \controller \in \controllerSet \label{eq:hlm_opt_pc_0_1_constraints} \\
    & \sum_{(\abstractState', \controller') \in pred(\abstractSuccessState)} \occupancyVar(\abstractState', \controller') \bernoulliProb_{\controller'} \geq 1 - \hlmFailProb
    \label{eq:hlm_opt_task_sat_constraints}
\end{align}%
The decision variables in (\ref{eq:hlm_opt_objective})-(\ref{eq:hlm_opt_task_sat_constraints}) are the HLM parameters \(\bernoulliProb_{\controller}\) for every \(\controller \in \controllerSet\), and \(\occupancyVar(\abstractState, \controller)\) for every \(\abstractState \in \abstractStateSet \setminus \{\abstractFailureState, \abstractSuccessState\}\).
The value of \(\delta_{\abstractInitialState}(\abstractState)\) is $1$ if \(\abstractState=\abstractInitialState\) and $0$ otherwise.
The constraint~\eqref{eq:hlm_opt_dynamics_constraints} is the so-called Bellman-flow constraint; it ensures that the variable \(\occupancyVar(\abstractState, \controller)\) defines the expected number of times subsystem \(\controller\) is executed in state \(\abstractState\).
The constraint \eqref{eq:hlm_opt_task_sat_constraints} enforces the HLM policy \(\hlmPolicy\)'s satisfaction of \(\mathbb{P}^{\abstractInitialState}_{\abstractMDP, \hlmPolicy}(\Diamond \abstractSuccessState) \geq 1 - \hlmFailProb\).
We refer to~\citet{etessami2007multi} and~\citet{puterman2014markov} for further details on these variables and the constraints.
\section{Iterative Compositional Reinforcement Learning (ICRL)}
\label{sec:icrl}

In this section, we discuss how subsystem policies are learned to satisfy the subtask specifications discussed above, and we present how the bilinear program given in (\ref{eq:hlm_opt_objective})-(\ref{eq:hlm_opt_task_sat_constraints}) is modified to refine the subtask specifications, after some training of the subsystems has been completed.

\subsubsection{Learning and Verifying Subsystem Policies.}
Let \(\bernoulliProb_{\controller_1}, \bernoulliProb_{\controller_2}, ..., \bernoulliProb_{\controller_{\numControllers}}\) be the parameter values output as a solution to problem (\ref{eq:hlm_opt_objective})-(\ref{eq:hlm_opt_task_sat_constraints}). We want each subsystem \(\controller\) to learn a policy \(\policy_{\controller}\) satisfying the subtask specification:  \(\successProb_{\policy_{\controller}}^{\controller}(\mdpState) \geq \bernoulliProb_{\controller}\) for each entry condition \(\mdpState \in \controllerInitialStateSet_{\controller}\) of the subsystem. 
We note that any RL algorithm and reward function may be used, so long as the resulting learned policy can be verified to satisfy its subtask specification.
A particularly simple candidate reward function \(\mdpRewardFunction_{\controller}\) outputs \(1\) when an exit condition \(\mdpState \in \controllerFinalStateSet_{\controller}\) is first reached, and outputs \(0\) otherwise. 
Under this reward function, we have  \(\successProb_{\policy_{\controller}}^{\controller}(\mdpState) = \mathbb{E}[\sum_{t\in[\timeHorizon_{\controller}]}\mdpRewardFunction_{\controller}(\mdpState_t) | \policy_{\controller}, \mdpState_0 = \mdpState]\). We can maximize the probability of reaching an exit condition by maximizing the expected undiscounted sum of rewards.

To verify that a learned subsystem policy \(\policy_{\controller}\) satisfies its subtask specification, we consider \(\controllerInfProb_{\controller} = \inf \{ \successProb_{\policy_{\controller}}^{\controller}(\mdpState)|\mdpState \in \controllerInitialStateSet_{\controller}\}\), the greatest lower bound of the policy's probability of subtask succcess, beginning from any of the subsystem's entry conditions. So long as \(\controllerInfProb_{\controller} \geq \bernoulliProb_{\controller}\), the subtask specification is satisfied.
In practice, the value of \(\controllerInfProb_{\controller}\) cannot be known exactly, but we may obtain an estimate \(\controllerPerformanceEstimate_{\controller}\) of its value through empirical rollouts of \(\policy_{\controller}\), beginning from the different entry conditions \(\mdpState \in \controllerInitialStateSet_{\controller}\).
We note that one may additionally use Hoeffding's inequality to obtain a high-confidence range of values for \(\controllerInfProb_{\controller}\), given the number of rollouts used.
We refer to \(\controllerPerformanceEstimate_{\controller}\) as the \textit{estimated performance value} of policy \(\policy_{\controller}\).

\subsubsection{Automatic Refinement of the Subtask Specifications.}
The estimated performance values \(\controllerPerformanceEstimate_{\controller}\) are useful not only for the empirical verification of the learned policies, but also as additional information used periodically during training to refine the subtask specifications. To do so, we re-solve the optimization problem (\ref{eq:hlm_opt_objective})-(\ref{eq:hlm_opt_task_sat_constraints}), with a modified objective (\ref{eq:hlm_opt_modified_obj}), and additional constraints \eqref{eq:hlm_opt_lb_constraints}-\eqref{eq:hlm_opt_ub_constraints}.%
\begin{align}
    & obj(\lbList) = \sum_{\controller \in \controllerSet} (\bernoulliProb_{\controller} - \controllerPerformanceEstimate_{\controller})
    \label{eq:hlm_opt_modified_obj}\\
    & LBConst(\lbList) = \{\bernoulliProb_{\controller} \geq \controllerPerformanceEstimate_{\controller} | \forall \controllerPerformanceEstimate_{\controller} \in \lbList\}
    \label{eq:hlm_opt_lb_constraints}\\
    & UBConst(\ubList) = \{\bernoulliProb_{\controller} \leq \controllerPerformanceEstimate_{\controller} | \forall \controllerPerformanceEstimate_{\controller} \in \ubList\}
    \label{eq:hlm_opt_ub_constraints}
\end{align}
Here, we assume that the subsystems have learned policies \(\policy_{\controller_1}, \policy_{\controller_2}, ..., \policy_{\controller_{\numControllers}}\). Let \(\lbList = \{\controllerPerformanceEstimate_{\controller_1}, \controllerPerformanceEstimate_{\controller_2}, ..., \controllerPerformanceEstimate_{\controller_{\numControllers}}\}\) be the set of the corresponding estimated performance values.
The objective function (\ref{eq:hlm_opt_modified_obj}) minimizes the performance gap between the subtask specifications \(\bernoulliProb_{\controller}\) and the current estimated performance values \(\controllerPerformanceEstimate_{\controller}\).
The rationale behind the additional constraints defined by \(LBConst(\lbList)\) is as follows: the subsystems have already learned policies achieving probabilities of subtask success greater than the estimated performance values \(\controllerPerformanceEstimate_{\controller}\), and so there is no reason to consider subtask specifications \(\bernoulliProb_{\controller}\) that are below these values.

\begin{algorithm}[t]
    \DontPrintSemicolon 
    \KwIn{ Partially instantiated subsystems \(\controllerSet = \{\controller_1, \controller_2, ..., \controller_{\numControllers}\}\), \(\hlmFailProb\), \(\trainingSteps\), \(\maxTrainingSteps\).}
    \KwOut{Subsystem policies \(\{\policy_{\controller_1}, \policy_{\controller_2}, ..., \policy_{\controller_{\numControllers}}\}\), meta-policy \(\abstractPolicy\), success probability \(\controllerPerformanceEstimate_{\abstractPolicy}\).}
    
    \(\abstractMDP \gets ConstructHLM(\controllerSet)\)\;
    
    

    \(\controllerPerformanceEstimate_{\controller_1}, \controllerPerformanceEstimate_{\controller_2}, ..., \controllerPerformanceEstimate_{\controller_{\numControllers}}, \controllerPerformanceEstimate_{\abstractPolicy} \gets 0\); \(\numSteps_{\controller_1}, \numSteps_{\controller_2}, ..., \numSteps_{\controller_{\numControllers}} \gets 0\)\;
    
    \(\lbList \gets \{\controllerPerformanceEstimate_{\controller_1}, \controllerPerformanceEstimate_{\controller_2}, ..., \controllerPerformanceEstimate_{\controller_{\numControllers}}\}\); \(\ubList \gets \{\}\)\;
    
    \While{\(\controllerPerformanceEstimate_{\abstractPolicy} \leq 1 - \delta\)}{
    
        \If{~\eqref{eq:hlm_opt_objective}-\eqref{eq:hlm_opt_ub_constraints} {\normalfont infeasible}}{
            \Return{{\normalfont Problem is infeasible.}}
        }
        
        \(\{\bernoulliProb_{\controller_1}, \ldots, \bernoulliProb_{\controller_{\numControllers}}\} \gets\)
        \text{Solve~\eqref{eq:hlm_opt_objective}-\eqref{eq:hlm_opt_ub_constraints} using} \((\abstractMDP, \lbList, \ubList)\)\;
        
        \(\controller_{j} \gets selectSubSystem(\bernoulliProb_{\controller_1}, \ldots, \bernoulliProb_{\controller_{\numControllers}}, \controllerPerformanceEstimate_{\controller_1},\ldots, \controllerPerformanceEstimate_{\controller_{\numControllers}})\)\;
    
        \(\policy_{\controller_{j}} \gets RLTrain(\controller_j, \policy_{\controller_j}, \trainingSteps)\); \(\numSteps_{\controller_j} \gets \numSteps_{\controller_j} + \trainingSteps\)\;
        
        
        \(\controllerPerformanceEstimate_{\controller_j} \gets estimateSubTaskSuccessProb(\controller_j, \policy_{\controller_j})\)\;
        
        \(\lbList.update(\controllerPerformanceEstimate_{\controller_j})\)\;
        
        \If{\(\numSteps_{\controller_j} \geq \maxTrainingSteps\)}{\(\ubList.add(\controllerPerformanceEstimate_{\controller_j})\)\;}
        
        \(\abstractPolicy \gets solveOptimalHLMPolicy(\abstractMDP, \lbList)\)\;
        
        \(\controllerPerformanceEstimate_{\abstractPolicy} \gets predictTaskSuccessProbability(\abstractMDP, \abstractPolicy, \lbList)\)\;
    }
    
    \Return{\(\{\policy_{\controller_1}, \policy_{\controller_2}, ..., \policy_{\controller_{\numControllers}}\}\), \(\abstractPolicy\), \(\controllerPerformanceEstimate_{\abstractPolicy}\)}\;
    
    \caption{Iterative Compositional RL (ICRL)}
    \label{alg:ICRL}
\end{algorithm}

Conversely, if the RL algorithm of a particular subsystem \(\controller\) has \textit{converged} -- i.e. the value of \(\controllerPerformanceEstimate_{\controller}\) will no longer increase with additional training steps -- we add the constraint \(\bernoulliProb_{\controller} \leq \controllerPerformanceEstimate_{\controller}\). 
This ensures that solutions to the optimization problem will \textit{not} yield a subtask specification \(\bernoulliProb_{\controller}\) that is larger than what the subsystem can realistically achieve. 
In practice, as a proxy to convergence, we allow each subsystem a maximum budget of \(\maxTrainingSteps\) training steps. 
Once any subsystem \(\controller\) has exceeded this training budget, we append \(\controllerPerformanceEstimate_{\controller}\) to the set \(\ubList\), which is used to define \(UBConst(\ubList)\) in (\ref{eq:hlm_opt_ub_constraints}).

\subsubsection{Iterative Compositional Reinforcement Learning (ICRL).}
By alternating between the training of the subsystems and the refinement of the subtask specifications, we obtain Algorithm \ref{alg:ICRL}.
In lines \(1-3\), the HLM is constructed from the collection of partially instantiated subsystems \(\controllerSet\) and the subsystem policies are initialized.
The while loop in lines \(4-12\) is the main loop controlling the subtask specifications and training of the subsystems.
In line \(5\), the bilinear program \eqref{eq:hlm_opt_objective}-\eqref{eq:hlm_opt_ub_constraints} is solved to update the values of \(\bernoulliProb_{\controller}\). 
These values are used, along with the estimated performance values, to select a subsystem to train.
A simple selection scheme, is to choose the subsystem \(\controller_j\) maximizing the current performance gap between \(\bernoulliProb_{\controller_j}\) and \(\controllerPerformanceEstimate_{\controller_j}\).
In line \(7\), the subsystem is trained for \(\trainingSteps\) steps using the RL algorithm of choice.
The subsystem's initial state is sampled uniformly from its entry conditions during training.
Finally, in line \(12\), the HLM \(\abstractMDP\) and the current estimated performance values \(\lbList\) are used to plan a meta-policy \(\abstractPolicy\) maximizing the probability \(\controllerPerformanceEstimate_{\abstractPolicy}\) of reaching the HLM goal state \(\abstractSuccessState\).
This step uses standard MDP algorithms \cite{puterman2014markov}.

We note that the conditions in lines \(4\) and \(5\) ensure that the algorithm only terminates once a meta-policy that satisfies the task specification exists, or the optimization problem~\eqref{eq:hlm_opt_objective}-\eqref{eq:hlm_opt_ub_constraints} has become infeasible.
One of these two outcomes is guarateed to eventually occur.
In particular, by our construction of \(\ubList\) and the corresponding constraints in \eqref{eq:hlm_opt_ub_constraints}, the problem will become infeasible if all of the allotted subsystem training budgets \(\numSteps_{max}\) have been exhausted and a satisfactory meta-policy still does not exist.
In such circumstances the task designer may wish to lower \(\delta\), to increase \(N_{max}\), or to further decompose the task using additional subtasks.

\section{Numerical Examples}

\label{sec:experiments}

In this section, we present the results of applying the proposed framework to the labyrinth navigation task used as a running example throughout the paper. 
We begin by discussing the results obtained using a discrete gridworld implementation of the labyrinth.
However, to help demonstrate the framework's generality, we also present results for a continuous-state and continuous-action labyrinth, whose dynamics are goverened by a rigid-body physics simulator.
Project code is available at: \href{https://github.com/cyrusneary/verifiable-compositional-rl}{github.com/cyrusneary/verifiable-compositional-rl}.

Figure~\ref{fig:labyrinth_gridworld} illustrates the labyrinth environment, and highlights each subtask with a different color, matching the colors used to represent the different subtasks in the presentation of the numerical results.
Recall that the overall task specification is to safely navigate from the labyrinth's initial state in the top left corner to the goal state marked by a green square in the bottom left corner, with a probability of at least \(0.95\). 

\subsubsection{Discrete Gridworld Labyrinth Environment.}
We implement the gridworld labyrinth environment using MiniGrid \cite{gym_minigrid}.
The environment's state space consists of the current position and orientation within the labyrinth, resulting in 1600 total states.
The allowed actions are: \textit{turn left}, \textit{turn right}, and \textit{move forward}.
A slip probability is added to the environment dynamics to render them stochastic; each action has a \(10\%\) probability of accidentally causing the result of a different action to occur.
Subtask entry \(\controllerInitialStateSet_{\controller}\) and exit \(\controllerFinalStateSet_{\controller}\) conditions are implemented as finite collections of states.

\subsubsection{ICRL Algorithm Implementation.}
Each RL subsystem is trained using the Stable-Baselines3 \cite{stable-baselines3} implementation of the proximal policy optimization (PPO) algorithm \cite{schulman2017proximal}.
Whenever estimates of task or subtask success probabilities are needed, we roll out the corresponding (sub)system 300 times from initial states randomly sampled from \(\controllerInitialStateSet_{\controller}\), and compute the empirical success rate.
We solve the bilinear program in~\eqref{eq:hlm_opt_objective}--\eqref{eq:hlm_opt_task_sat_constraints} using \textit{Gurobi}~\cite{gurobi}.
Gurobi transforms the bilinear program into an equivalent mixed-integer linear program, and computes a globally optimal solution to this program by using cutting plane and branch and bound methods.
For further details please see the supplementary material.

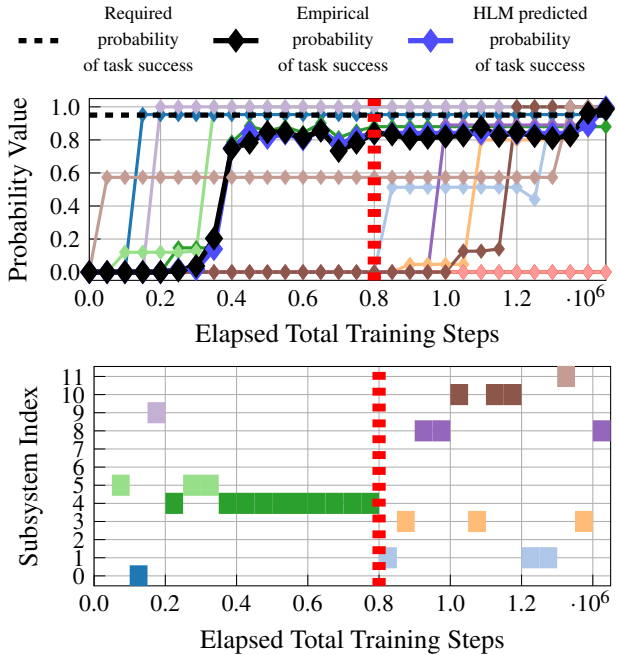
\begin{figure}[t]%
    \begin{subfigure}[t]{\columnwidth}
        \centering
        \begin{tikzpicture}
\definecolor{color0}{rgb}{0.12156862745098,0.466666666666667,0.705882352941177}
\definecolor{color1}{rgb}{0.682352941176471,0.780392156862745,0.909803921568627}
\definecolor{color2}{rgb}{1,0.498039215686275,0.0549019607843137}
\definecolor{color3}{rgb}{1,0.733333333333333,0.470588235294118}
\definecolor{color4}{rgb}{0.172549019607843,0.627450980392157,0.172549019607843}
\definecolor{color5}{rgb}{0.596078431372549,0.874509803921569,0.541176470588235}
\definecolor{color6}{rgb}{0.83921568627451,0.152941176470588,0.156862745098039}
\definecolor{color7}{rgb}{1,0.596078431372549,0.588235294117647}
\definecolor{color8}{rgb}{0.580392156862745,0.403921568627451,0.741176470588235}
\definecolor{color9}{rgb}{0.772549019607843,0.690196078431373,0.835294117647059}
\definecolor{color10}{rgb}{0.549019607843137,0.337254901960784,0.294117647058824}
\definecolor{color11}{rgb}{0.768627450980392,0.611764705882353,0.580392156862745}

\begin{customlegend}[legend columns=3, legend style={align=center, column sep=1ex, font=\scriptsize, draw=none}, legend entries={Required\\ probability\\of task success, Empirical\\ probability\\of task success, HLM predicted\\ probability\\of task success}]
\addlegendimage{line width=2pt, black, dashed}
\addlegendimage{line width=2pt, black, mark=diamond*, mark size=3}
\addlegendimage{line width=2pt, blue!70!white, mark=diamond*, mark size=3}
\end{customlegend}
\end{tikzpicture}
    \end{subfigure}\hfill%
    \centering
    \begin{subfigure}[t]{0.48\textwidth}
        \centering
\begin{tikzpicture}

\definecolor{color0}{rgb}{0.12156862745098,0.466666666666667,0.705882352941177}
\definecolor{color1}{rgb}{0.682352941176471,0.780392156862745,0.909803921568627}
\definecolor{color2}{rgb}{1,0.498039215686275,0.0549019607843137}
\definecolor{color3}{rgb}{1,0.733333333333333,0.470588235294118}
\definecolor{color4}{rgb}{0.172549019607843,0.627450980392157,0.172549019607843}
\definecolor{color5}{rgb}{0.596078431372549,0.874509803921569,0.541176470588235}
\definecolor{color6}{rgb}{0.83921568627451,0.152941176470588,0.156862745098039}
\definecolor{color7}{rgb}{1,0.596078431372549,0.588235294117647}
\definecolor{color8}{rgb}{0.580392156862745,0.403921568627451,0.741176470588235}
\definecolor{color9}{rgb}{0.772549019607843,0.690196078431373,0.835294117647059}
\definecolor{color10}{rgb}{0.549019607843137,0.337254901960784,0.294117647058824}
\definecolor{color11}{rgb}{0.768627450980392,0.611764705882353,0.580392156862745}

\begin{axis}[
tick align=inside,
tick pos=left,
x grid style={white!69.0196078431373!black},
  ticklabel style = {font=\footnotesize},
xmajorgrids,
xmin=0, xmax=1450000,
xtick style={color=black},
y grid style={white!69.0196078431373!black},
ymajorgrids,
ymin=-0.05, ymax=1.05,
height=4.0cm,
width=0.99\textwidth,
ytick={0.0, 0.2, 0.4, 0.6, 0.8, 1.0},
yticklabels={0.0, 0.2, 0.4, 0.6, 0.8, 1.0},
xtick={0, 200000, 400000, 600000, 800000, 1000000, 1200000, 1400000},
xticklabels={0.0, 0.2, 0.4, 0.6, 0.8, 1.0, 1.2, },
ytick style={color=black},
legend to name=named,
every x tick scale label/.style={
    at={(1,0)},xshift=-15.5pt,yshift=-10.0pt,anchor=south west,inner sep=0pt},
ylabel={Probability Value},
xlabel={Elapsed Total Training Steps}
]
\addplot [very thick, color0, mark=diamond*, mark size=2]
table {%
0 0
50000 0
100000 0
150000 0.953333333333333
200000 0.953333333333333
250000 0.953333333333333
300000 0.953333333333333
350000 0.953333333333333
400000 0.953333333333333
450000 0.953333333333333
500000 0.953333333333333
550000 0.953333333333333
600000 0.953333333333333
650000 0.953333333333333
700000 0.953333333333333
750000 0.953333333333333
800000 0.953333333333333
850000 0.953333333333333
900000 0.953333333333333
950000 0.953333333333333
1000000 0.953333333333333
1050000 0.953333333333333
1100000 0.953333333333333
1150000 0.953333333333333
1200000 0.953333333333333
1250000 0.953333333333333
1300000 0.953333333333333
1350000 0.953333333333333
1400000 0.953333333333333
1450000 0.953333333333333
};
\addlegendentry{1}
\addplot [very thick, color1, mark=diamond*, mark size=2]
table {%
0 0
50000 0
100000 0
150000 0
200000 0
250000 0
300000 0
350000 0
400000 0
450000 0
500000 0
550000 0
600000 0
650000 0
700000 0
750000 0
800000 0
850000 0.513333333333333
900000 0.513333333333333
950000 0.513333333333333
1000000 0.513333333333333
1050000 0.513333333333333
1100000 0.513333333333333
1150000 0.513333333333333
1200000 0.513333333333333
1250000 0.44
1300000 1
1350000 1
1400000 1
1450000 1
};
\addplot [very thick, color2, mark=diamond*, mark size=2]
table {%
0 0
50000 0
100000 0
150000 0
200000 0
250000 0
300000 0
350000 0
400000 0
450000 0
500000 0
550000 0
600000 0
650000 0
700000 0
750000 0
800000 0
850000 0
900000 0
950000 0
1000000 0
1050000 0
1100000 0
1150000 0
1200000 0
1250000 0
1300000 0
1350000 0
1400000 0
1450000 0
};
\addplot [very thick, color3, mark=diamond*, mark size=2]
table {%
0 0
50000 0
100000 0
150000 0
200000 0
250000 0
300000 0
350000 0
400000 0
450000 0
500000 0
550000 0
600000 0
650000 0
700000 0
750000 0
800000 0
850000 0
900000 0.0466666666666667
950000 0.0466666666666667
1000000 0.0466666666666667
1050000 0.0466666666666667
1100000 0.8
1150000 0.8
1200000 0.8
1250000 0.8
1300000 0.8
1350000 0.8
1400000 0.996666666666667
1450000 0.996666666666667
};
\addplot [very thick, color4, mark=diamond*, mark size=2]
table {%
0 0
50000 0
100000 0
150000 0
200000 0
250000 0.146666666666667
300000 0.146666666666667
350000 0.146666666666667
400000 0.786666666666667
450000 0.883333333333333
500000 0.853333333333333
550000 0.88
600000 0.843333333333333
650000 0.9
700000 0.813333333333333
750000 0.853333333333333
800000 0.88
850000 0.88
900000 0.88
950000 0.88
1000000 0.88
1050000 0.88
1100000 0.88
1150000 0.88
1200000 0.88
1250000 0.88
1300000 0.88
1350000 0.88
1400000 0.88
1450000 0.88
};
\addplot [very thick, color5, mark=diamond*, mark size=2]
table {%
0 0
50000 0
100000 0.12
150000 0.12
200000 0.12
250000 0.12
300000 0.126666666666667
350000 1
400000 1
450000 1
500000 1
550000 1
600000 1
650000 1
700000 1
750000 1
800000 1
850000 1
900000 1
950000 1
1000000 1
1050000 1
1100000 1
1150000 1
1200000 1
1250000 1
1300000 1
1350000 1
1400000 1
1450000 1
};
\addplot [very thick, color6, mark=diamond*, mark size=2]
table {%
0 0
50000 0
100000 0
150000 0
200000 0
250000 0
300000 0
350000 0
400000 0
450000 0
500000 0
550000 0
600000 0
650000 0
700000 0
750000 0
800000 0
850000 0
900000 0
950000 0
1000000 0
1050000 0
1100000 0
1150000 0
1200000 0
1250000 0
1300000 0
1350000 0
1400000 0
1450000 0
};
\addplot [very thick, color7, mark=diamond*, mark size=2]
table {%
0 0
50000 0
100000 0
150000 0
200000 0
250000 0
300000 0
350000 0
400000 0
450000 0
500000 0
550000 0
600000 0
650000 0
700000 0
750000 0
800000 0
850000 0
900000 0
950000 0
1000000 0
1050000 0
1100000 0
1150000 0
1200000 0
1250000 0
1300000 0
1350000 0
1400000 0
1450000 0
};
\addplot [very thick, color8, mark=diamond*, mark size=2]
table {%
0 0
50000 0
100000 0
150000 0
200000 0
250000 0
300000 0
350000 0
400000 0
450000 0
500000 0
550000 0
600000 0
650000 0
700000 0
750000 0
800000 0
850000 0
900000 0
950000 0
1000000 0.89
1050000 0.89
1100000 0.89
1150000 0.89
1200000 0.89
1250000 0.89
1300000 0.89
1350000 0.89
1400000 0.89
1450000 1
};
\addplot [very thick, color9, mark=diamond*, mark size=2]
table {%
0 0
50000 0
100000 0
150000 0
200000 1
250000 1
300000 1
350000 1
400000 1
450000 1
500000 1
550000 1
600000 1
650000 1
700000 1
750000 1
800000 1
850000 1
900000 1
950000 1
1000000 1
1050000 1
1100000 1
1150000 1
1200000 1
1250000 1
1300000 1
1350000 1
1400000 1
1450000 1
};
\addplot [very thick, color10, mark=diamond*, mark size=2]
table {%
0 0
50000 0
100000 0
150000 0
200000 0
250000 0
300000 0
350000 0
400000 0
450000 0
500000 0
550000 0
600000 0
650000 0
700000 0
750000 0
800000 0
850000 0
900000 0
950000 0
1000000 0
1050000 0.126666666666667
1100000 0.126666666666667
1150000 0.14
1200000 1
1250000 1
1300000 1
1350000 1
1400000 1
1450000 1
};
\addplot [very thick, color11, mark=diamond*, mark size=2]
table {%
0 0
50000 0.573333333333333
100000 0.573333333333333
150000 0.573333333333333
200000 0.573333333333333
250000 0.573333333333333
300000 0.573333333333333
350000 0.573333333333333
400000 0.573333333333333
450000 0.573333333333333
500000 0.573333333333333
550000 0.573333333333333
600000 0.573333333333333
650000 0.573333333333333
700000 0.573333333333333
750000 0.573333333333333
800000 0.573333333333333
850000 0.573333333333333
900000 0.573333333333333
950000 0.573333333333333
1000000 0.573333333333333
1050000 0.573333333333333
1100000 0.573333333333333
1150000 0.573333333333333
1200000 0.573333333333333
1250000 0.573333333333333
1300000 0.573333333333333
1350000 1
1400000 1
1450000 1
};
\addplot [line width=2pt, black, dashed ,on layer=foreground]
table {%
0 0.95
1450000 0.95
};
\addplot [line width=2pt, blue!70!white, mark=diamond*, mark size=3]
table {%
0 0
50000 0
100000 0
150000 0
200000 0
250000 0.0167786666666667
300000 0.0177108148148148
350000 0.139822222222222
400000 0.749955555555555
450000 0.842111111111111
500000 0.813511111111111
550000 0.838933333333333
600000 0.803977777777778
650000 0.858
700000 0.775377777777778
750000 0.813511111111111
800000 0.838933333333333
850000 0.838933333333333
900000 0.838933333333333
950000 0.838933333333333
1000000 0.838933333333333
1050000 0.838933333333333
1100000 0.838933333333333
1150000 0.838933333333333
1200000 0.838933333333333
1250000 0.838933333333333
1300000 0.838933333333333
1350000 0.838933333333333
1400000 0.887033333333333
1450000 0.996666666666667
};
\addplot [line width=2pt, black, mark=diamond*, mark size=3]
table {%
0 0
50000 0
100000 0
150000 0
200000 0
250000 0.01
300000 0.0366666666666667
350000 0.203333333333333
400000 0.746666666666667
450000 0.783333333333333
500000 0.846666666666667
550000 0.846666666666667
600000 0.813333333333333
650000 0.856666666666667
700000 0.733333333333333
750000 0.783333333333333
800000 0.84
850000 0.833333333333333
900000 0.806666666666667
950000 0.813333333333333
1000000 0.82
1050000 0.826666666666667
1100000 0.876666666666667
1150000 0.823333333333333
1200000 0.846666666666667
1250000 0.823333333333333
1300000 0.81
1350000 0.826666666666667
1400000 0.963333333333333
1450000 0.986666666666667
};
\addlegendentry{Empirically Measured Probability of Task Success}
\addplot [line width=5pt, red, dashed  ,on layer=foreground]
table {%
800000 -0.05
800000 1.05
};
\end{axis}
\end{tikzpicture}
        \label{fig:experiments:b}
    \end{subfigure}
    \hspace*{\fill}
    \centering
    \begin{subfigure}[t]{0.48\textwidth}
        \centering
\begin{tikzpicture}

\definecolor{color0}{rgb}{0.596078431372549,0.874509803921569,0.541176470588235}
\definecolor{color1}{rgb}{0.12156862745098,0.466666666666667,0.705882352941177}
\definecolor{color2}{rgb}{0.772549019607843,0.690196078431373,0.835294117647059}
\definecolor{color3}{rgb}{0.172549019607843,0.627450980392157,0.172549019607843}
\definecolor{color4}{rgb}{0.682352941176471,0.780392156862745,0.909803921568627}
\definecolor{color5}{rgb}{1,0.733333333333333,0.470588235294118}
\definecolor{color6}{rgb}{0.580392156862745,0.403921568627451,0.741176470588235}
\definecolor{color7}{rgb}{0.549019607843137,0.337254901960784,0.294117647058824}
\definecolor{color8}{rgb}{0.768627450980392,0.611764705882353,0.580392156862745}

\begin{axis}[
height=4.5cm,
width=0.99\textwidth,
tick align=inside,
tick pos=left,
x grid style={white!69.0196078431373!black},
ticklabel style = {font=\footnotesize},
xmajorgrids,
xmin=0, xmax=1450000,
xtick style={color=black},
y grid style={white!69.0196078431373!black},
ymajorgrids,
ymin=-0.55, ymax=11.55,
ytick style={color=black},
ytick={0,1,2,3,4,5,6,7,8,9,10,11},
yticklabels = {0,1,2,3,4,5,6,7,8,9,10,11},
xtick={0, 200000, 400000, 600000, 800000, 1000000, 1200000, 1400000},
xticklabels={0.0, 0.2, 0.4, 0.6, 0.8, 1.0, 1.2, },
every x tick scale label/.style={
    at={(1,0)},xshift=-15.5pt,yshift=-10.0pt,anchor=south west,inner sep=0pt},
ylabel={Subsystem Index},
xlabel={Elapsed Total Training Steps}
]
\addplot [line width=8pt, color0]
table {%
50000 5
100000 5
};
\addplot [line width=8pt, color1]
table {%
100000 0
150000 0
};
\addplot [line width=8pt, color2]
table {%
150000 9
200000 9
};
\addplot [line width=8pt, color3]
table {%
200000 4
250000 4
};
\addplot [line width=8pt, color0]
table {%
250000 5
300000 5
};
\addplot [line width=8pt, color0]
table {%
300000 5
350000 5
};
\addplot [line width=8pt, color3]
table {%
350000 4
400000 4
};
\addplot [line width=8pt, color3]
table {%
400000 4
450000 4
};
\addplot [line width=8pt, color3]
table {%
450000 4
500000 4
};
\addplot [line width=8pt, color3]
table {%
500000 4
550000 4
};
\addplot [line width=8pt, color3]
table {%
550000 4
600000 4
};
\addplot [line width=8pt, color3]
table {%
600000 4
650000 4
};
\addplot [line width=8pt, color3]
table {%
650000 4
700000 4
};
\addplot [line width=8pt, color3]
table {%
700000 4
750000 4
};
\addplot [line width=8pt, color3]
table {%
750000 4
800000 4
};
\addplot [line width=8pt, color4]
table {%
800000 1
850000 1
};
\addplot [line width=8pt, color5]
table {%
850000 3
900000 3
};
\addplot [line width=8pt, color6]
table {%
900000 8
950000 8
};
\addplot [line width=8pt, color6]
table {%
950000 8
1000000 8
};
\addplot [line width=8pt, color7]
table {%
1000000 10
1050000 10
};
\addplot [line width=8pt, color5]
table {%
1050000 3
1100000 3
};
\addplot [line width=8pt, color7]
table {%
1100000 10
1150000 10
};
\addplot [line width=8pt, color7]
table {%
1150000 10
1200000 10
};
\addplot [line width=8pt, color4]
table {%
1200000 1
1250000 1
};
\addplot [line width=8pt, color4]
table {%
1250000 1
1300000 1
};
\addplot [line width=8pt, color8]
table {%
1300000 11
1350000 11
};
\addplot [line width=8pt, color5]
table {%
1350000 3
1400000 3
};
\addplot [line width=8pt, color6]
table {%
1400000 8
1450000 8
};
\addplot [line width=5pt, red, dashed]
table {%
800000 -0.55
800000 11.55
};
\end{axis}

\end{tikzpicture}
    \label{fig:experiments:c}
    \end{subfigure}
    \caption{Discrete labyrinth experimental results.
    Top: Estimated task and subtask success probabilities during training.
    Bottom: Automatically generated subsystem training schedule.
    Each subtask is represented by a different color, matching those used in Figure \ref{fig:labyrinth_gridworld}.
    The dotted red lines illustrate the point in training at which the HLM automatically refines the subtask specifications.
    Step counts do not include the rollouts used to estimate subtask success probabilities.
    }
    \label{fig:experiments}
\end{figure}

\definecolor{color0}{rgb}{0.12156862745098,0.466666666666667,0.705882352941177}
\definecolor{color1}{rgb}{0.682352941176471,0.780392156862745,0.909803921568627}
\definecolor{color2}{rgb}{1,0.498039215686275,0.0549019607843137}
\definecolor{color3}{rgb}{1,0.733333333333333,0.470588235294118}
\definecolor{color4}{rgb}{0.172549019607843,0.627450980392157,0.172549019607843}
\definecolor{color5}{rgb}{0.596078431372549,0.874509803921569,0.541176470588235}
\definecolor{color6}{rgb}{0.83921568627451,0.152941176470588,0.156862745098039}
\definecolor{color7}{rgb}{1,0.596078431372549,0.588235294117647}
\definecolor{color8}{rgb}{0.580392156862745,0.403921568627451,0.741176470588235}
\definecolor{color9}{rgb}{0.772549019607843,0.690196078431373,0.835294117647059}
\definecolor{color10}{rgb}{0.549019607843137,0.337254901960784,0.294117647058824}
\definecolor{color11}{rgb}{0.768627450980392,0.611764705882353,0.580392156862745}

\definecolor{highlightgrey}{gray}{0.85}

\begin{table*}[t]
\centering
 \begin{tabular}{|c | c | c | c | c | c | c | c | c | c | c | c | c |} 
 \hline
 Subsystem Index & 
 \cellcolor{color0}0 & 
 \cellcolor{color1}1 & 
 \cellcolor{color2}2 & 
 \cellcolor{color3}3 & 
 \cellcolor{color4}4 & 
 \cellcolor{color5}5 & 
 \cellcolor{color6}6 & 
 \cellcolor{color7}7 & 
 \cellcolor{color8}8 & 
 \cellcolor{color9}9 & 
 \cellcolor{color10}10 & 
 \cellcolor{color11}11 \\
 \hline\hline
 \(\bernoulliProb_{\controller}\) at \(t = 6e5\) 
 & \cellcolor{highlightgrey}.97 
 & .00 
 & .00 
 & .00 
 & \cellcolor{highlightgrey}.97 
 & \cellcolor{highlightgrey}1.0 
 & .00 
 & .00 
 & .00 
 & \cellcolor{highlightgrey}1.0 
 & .00 
 & .57\\ 
 \hline
 \(\bernoulliProb_{\controller}\) at \(t = 10e5\) 
 & .95 
 & \cellcolor{highlightgrey}.99 
 & .00 
 & \cellcolor{highlightgrey}.99 
 & .88 
 & 1.0 
 & .00 
 & .00 
 & \cellcolor{highlightgrey}.99 
 & 1.0 
 & \cellcolor{highlightgrey}.99 
 & \cellcolor{highlightgrey}.99\\
 \hline
\end{tabular}
\caption{\label{tab:sub_task_specifications} Demonstration of automatic subtask specification refinement. 
Each value corresponds to a subtask specification, i.e. the minimum allowable probability of subtask success. 
The two rows of the table show these values at two distinct points of the system's training; before and after the subtask specification refinement illustrated by the dotted red lines in Figure \ref{fig:experiments}. 
The cells highlighted in grey indicate which subsystems are used by the meta-policy, at the specified point.}
\end{table*}

\subsubsection{Empirical Validation of Theorem \ref{thm:hlm_bounds_true_performance}.}
At regular intervals during training, marked by diamonds in Figure \ref{fig:experiments}, each subsystem's probability of subtask success is estimated and used to update \(\lbList\) and \(\ubList\), as described in the previous section. That is, each diamond in Figure \ref{fig:experiments} corresponds to a pass through the main loop of algorithm \ref{alg:ICRL}.
The HLM-predicted probability of the meta-policy completing the overall task is illustrated in Figure \ref{fig:experiments} by the navy blue curve.
For comparison, we plot empirical measurements of the success rate of the meta-policy in black.
We clearly observe that the HLM predictions closely match the empirical measurements. 

\subsubsection{Subtask Specification Refinements Lead to Meta-Policy Adaptation and Targeted Subsystem Training.}
Figure \ref{fig:experiments} illustrates the subsystem training schedule. Table \ref{tab:sub_task_specifications} lists the values of \(\bernoulliProb_{\controller}\) for each subsystem \(\controller\).
We observe from Table \ref{tab:sub_task_specifications} that prior to \(8e5\) elapsed training steps, the value of \(\bernoulliProb_{\controller}\) is only specified to be close to \(1.0\) for subsystems \(\controller_0\), \(\controller_4\), \(\controller_5\), and, \(\controller_9\). 
As can be seen in Figure \ref{fig:labyrinth_gridworld}, these are the subsystems needed to move straight down, through the rooms containing lava, to the goal. 
The HLM has selected a meta-policy that will only use these subsystems because their composition yields the shortest path to goal; this path only requires training of 4 of the subsystems.
Furthermore, because the meta-policy does not use any of the other subsystems, it places no requirements on their probability of subtask success.
Figure \ref{fig:experiments} agrees with this observation: only this small collection of the subsystems are trained prior to \(8e5\) elapsed training steps.
In particular, subsystem \(4\), which must navigate the top lava room and is represented by dark green, is trained extensively.
However, due to the environment slip probability, this subsystem is unable to meet its subtask specification, \textit{safely navigate to the room's exit with probability 0.97}, regardless of the number of training iterations it receives.
As a result, subsystem \(4\) exhausts its individual training budget after \(8e5\) elapsed system training steps, marked by the vertical dotted red lines in Figure \ref{fig:experiments}.
At this point, subsystem \(4\)'s empirically estimated success rate of 0.88 is used to update the HLM, which then refines the subtask specifications.
The result of this refinement is a new meta-policy, which instead uses subsystems \(\controller_1\), \(\controller_3\), \(\controller_8\), \(\controller_{10}\), and \(\controller_{11}\) to take an alternate path that avoids the lava rooms altogether.
The updated subtask specifications are listed in the second row of Table \ref{tab:sub_task_specifications}, and in Figure \ref{fig:experiments} we observe a distinct change in the subsystems that are trained.
Once subsystems \(\controller_1\), \(\controller_3\), \(\controller_8\), \(\controller_{10}\), and \(\controller_{11}\) learn to satisfy their new subtask specifications with the required probability, the composite system's probability of task success rises above \(0.95\), satisfying the overall task specification.

\subsubsection{Comparison to a Monolithic RL Approach.}
The proposed ICRL algorithm takes less than two million training steps to satisfy the task specification.
By comparison, a monolithic approach in which the entire task is treated as a single subsystem takes roughly thirty million training steps.
We note that this is not a fair comparison because the proposed compositional approach has a priori knowledge of the subsystem entry and exit conditions.
However, such information is often available through natural decompositions of complex systems.
The proposed framework provides a method to take advantage of such information when it is available.

\subsubsection{Results in a Continuous Labyrinth Environment.}
To demonstrate the framework's ability to generalize to different RL settings, we also implemented a continuous-state and continuous-action version of the labyrinth environment in the video game engine \textit{Unity} \citep{juliani2018unity}.
In this version of the task, the RL system must roll a ball from the initial location to the goal location.
The set \(\mdpActionSet\) of available actions consists of all of the force vectors, with magnitude of at most 1, that can be applied to the ball in the horizontal plane.
The set \(\mdpStateSet\) of environment states is given by all possible locations \((x,y)\) and velocities \((\dot{x}, \dot{y})\) of the ball within the labyrinth.
The action space \(\mdpActionSet\) is thus a compact subset of \(\mathbb{R}^2\) while the state space \(\mdpStateSet\) is a compact subset of \(\mathbb{R}^4\). 
The transition dynamics are governed by \textit{Unity}'s rigid-body physics simulator.
Subtask entry \(\controllerInitialStateSet_{\controller}\) and exit \(\controllerFinalStateSet_{\controller}\) conditions are implemented as subsets of \(\mathbb{R}^4\) such that \(\sqrt{(x- x_c)^2 + (y-y_c)^2} \leq 0.5m\) and \(\sqrt{\dot{x}^2 + \dot{y}^2} \leq 0.5 \frac{m}{s}\) respectively, for some pre-specified \(x_c\) and \(y_c\).
We use the PPO algorithm to train the RL subsystem policies.
Each RL subsystem receives rewards that are proportional to its negative distance to the exit conditions, and incurs a large penalty if the lava is touched.
We refer to the supplementary material for additional details and figures of this continuous environment.

Figure \ref{fig:cont_lab_training_curves} illustrates the experimental results in the continuous labyrinth environment.
Qualitatively, these results closely resemble our observations from the discrete labyrinth, 
despite significant differences in the environment's dynamics and in its representations of states and actions.
The ICRL algorithm again initially attempts to move straight down past the lava, before automatically refining the subtask specifications in order to focus on training the subsystems that take the alternate route through the labyrinth.
This similarity in the algorithm's behavior when applied to different types of environments helps illustrate the generality of the proposed framework; ICRL is agnostic to the details of the environment dynamics and of the individual RL subsystems.

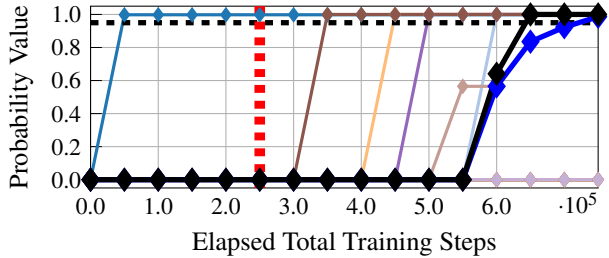
\begin{figure}[t]
    \begin{subfigure}{\columnwidth}
        \centering
\begin{tikzpicture}

\definecolor{color0}{rgb}{0.12156862745098,0.466666666666667,0.705882352941177}
\definecolor{color1}{rgb}{0.682352941176471,0.780392156862745,0.909803921568627}
\definecolor{color2}{rgb}{1,0.498039215686275,0.0549019607843137}
\definecolor{color3}{rgb}{1,0.733333333333333,0.470588235294118}
\definecolor{color4}{rgb}{0.172549019607843,0.627450980392157,0.172549019607843}
\definecolor{color5}{rgb}{0.596078431372549,0.874509803921569,0.541176470588235}
\definecolor{color6}{rgb}{0.83921568627451,0.152941176470588,0.156862745098039}
\definecolor{color7}{rgb}{1,0.596078431372549,0.588235294117647}
\definecolor{color8}{rgb}{0.580392156862745,0.403921568627451,0.741176470588235}
\definecolor{color9}{rgb}{0.772549019607843,0.690196078431373,0.835294117647059}
\definecolor{color10}{rgb}{0.549019607843137,0.337254901960784,0.294117647058824}
\definecolor{color11}{rgb}{0.768627450980392,0.611764705882353,0.580392156862745}

\begin{axis}[
tick align=inside,
tick pos=left,
x grid style={white!69.0196078431373!black},
  ticklabel style = {font=\footnotesize},
xmajorgrids,
xmin=0, xmax=750000,
xtick style={color=black},
y grid style={white!69.0196078431373!black},
ymajorgrids,
ymin=-0.05, ymax=1.05,
height=4.0cm,
width=0.99\textwidth,
ytick={0.0, 0.2, 0.4, 0.6, 0.8, 1.0},
yticklabels={0.0, 0.2, 0.4, 0.6, 0.8, 1.0},
xtick={0, 100000, 200000, 300000, 400000, 500000, 600000, 800000, 1000000, 1200000, 1400000},
xticklabels={0.0, 1.0, 2.0, 3.0, 4.0, 5.0, 6.0, 8.0, 10.0, 12.0, },
ytick style={color=black},
legend to name=named,
every x tick scale label/.style={
    at={(1,0)},xshift=-15.5pt,yshift=-10.0pt,anchor=south west,inner sep=0pt},
ylabel={Probability Value},
xlabel={Elapsed Total Training Steps}
]
\addplot [very thick, color0, mark=diamond*, mark size=2, mark options={solid}, forget plot]
table {%
0 0
0 0
50000 0.998
100000 0.998
150000 0.998
200000 0.998
250000 0.998
300000 0.998
350000 0.998
400000 0.998
450000 0.998
500000 0.998
550000 0.998
600000 0.998
650000 0.998
700000 0.998
750000 0.998
};
\addplot [very thick, color1, mark=diamond*, mark size=2, mark options={solid}, forget plot]
table {%
0 0
0 0
50000 0
100000 0
150000 0
200000 0
250000 0
300000 0
350000 0
400000 0
450000 0
500000 0
550000 0
600000 1
650000 1
700000 1
750000 1
};
\addplot [very thick, color2, mark=diamond*, mark size=2, mark options={solid}, forget plot]
table {%
0 0
0 0
50000 0
100000 0
150000 0
200000 0
250000 0
300000 0
350000 0
400000 0
450000 0
500000 0
550000 0
600000 0
650000 0
700000 0
750000 0
};
\addplot [very thick, color3, mark=diamond*, mark size=2, mark options={solid}, forget plot]
table {%
0 0
0 0
50000 0
100000 0
150000 0
200000 0
250000 0
300000 0
350000 0
400000 0
450000 1
500000 1
550000 1
600000 1
650000 1
700000 1
750000 1
};
\addplot [very thick, color4, mark=diamond*, mark size=2, mark options={solid}, forget plot]
table {%
0 0
0 0
50000 0
100000 0
150000 0
200000 0
250000 0
300000 0
350000 0
400000 0
450000 0
500000 0
550000 0
600000 0
650000 0
700000 0
750000 0
};
\addplot [very thick, color5, mark=diamond*, mark size=2, mark options={solid}, forget plot]
table {%
0 0
0 0
50000 0
100000 0
150000 0
200000 0
250000 0
300000 0
350000 0
400000 0
450000 0
500000 0
550000 0
600000 0
650000 0
700000 0
750000 0
};
\addplot [very thick, color6, mark=diamond*, mark size=2, mark options={solid}, forget plot]
table {%
0 0
0 0
50000 0
100000 0
150000 0
200000 0
250000 0
300000 0
350000 0
400000 0
450000 0
500000 0
550000 0
600000 0
650000 0
700000 0
750000 0
};
\addplot [very thick, color7, mark=diamond*, mark size=2, mark options={solid}, forget plot]
table {%
0 0
0 0
50000 0
100000 0
150000 0
200000 0
250000 0
300000 0
350000 0
400000 0
450000 0
500000 0
550000 0
600000 0
650000 0
700000 0
750000 0
};
\addplot [very thick, color8, mark=diamond*, mark size=2, mark options={solid}, forget plot]
table {%
0 0
0 0
50000 0
100000 0
150000 0
200000 0
250000 0
300000 0
350000 0
400000 0
450000 0
500000 1
550000 1
600000 1
650000 1
700000 1
750000 1
};
\addplot [very thick, color9, mark=diamond*, mark size=2, mark options={solid}, forget plot]
table {%
0 0
0 0
50000 0
100000 0
150000 0
200000 0
250000 0
300000 0
350000 0
400000 0
450000 0
500000 0
550000 0
600000 0
650000 0
700000 0
750000 0
};
\addplot [very thick, color10, mark=diamond*, mark size=2, mark options={solid}, forget plot]
table {%
0 0
0 0
50000 0
100000 0
150000 0
200000 0
250000 0
300000 0
350000 1
400000 1
450000 1
500000 1
550000 1
600000 1
650000 1
700000 1
750000 1
};
\addplot [very thick, color11, mark=diamond*, mark size=2, mark options={solid}, forget plot]
table {%
0 0
0 0
50000 0
100000 0
150000 0
200000 0
250000 0
300000 0
350000 0
400000 0
450000 0
500000 0
550000 0.565
600000 0.565
650000 0.835
700000 0.923
750000 0.986
};
\addplot [line width=2pt, black, dashed]
table {%
0 0.95
750000 0.95
};
\addlegendentry{Required Probability of Success}
\addplot [line width=2pt, blue, mark=diamond*, mark size=3, mark options={solid}]
table {%
0 0
0 0
50000 0
100000 0
150000 0
200000 0
250000 0
300000 0
350000 0
400000 0
450000 0
500000 0
550000 0
600000 0.565
650000 0.835
700000 0.923
750000 0.986
};
\addlegendentry{Lower Bound on Probability of Task Success}
\addplot [line width=2pt, black, mark=diamond*, mark size=3, mark options={solid}]
table {%
0 0
0 0
50000 0
100000 0
150000 0
200000 0
250000 0
300000 0
350000 0
400000 0
450000 0
500000 0
550000 0
600000 0.641
650000 1
700000 1
750000 1
};
\addlegendentry{Empirically Measured Probability of Task Success}
\addplot [line width=4pt, red, dashed, forget plot]
table {%
250000 -0.05
250000 1.05
};
\end{axis}

\end{tikzpicture}
    \end{subfigure}
    \caption{Continuous labyrinth experimental results.}
    \label{fig:cont_lab_training_curves}
\end{figure}

\subsubsection{Additional Discussion.}

We note that all predictions made using the HLM will be sensitive to the values of \(\controllerPerformanceEstimate_{\controller}\) -- the estimated lower bounds on the probability of subtask success.
In our experiments, we compute \(\controllerPerformanceEstimate_{\controller}\) empirically by rolling out the subsystems from randomly sampled entry conditions.
While this technique provides only rough estimates of the true value of the lower bound (particularly in the case of the continuous labyrinth environment which has an uncountably infinite number of entry conditions per subtask), our results demonstrate that these empirical approximations are sufficient for high-level decision making.
The algorithm makes effective use of the HLM predictions to automatically select the subsystems that require training.
Any methods to further improve the estimates of \(\controllerPerformanceEstimate_{\controller}\) will only improve the performance of the ICRL algorithm.
\section{Related Work}

While the proposed framework is closely related to hierarchical RL (HRL) \citep{sutton1999between, barto2003recent, kulkarni2016hierarchical, vezhnevets2017feudal, nachum2019data, levy2017learning}, our framework adds several benefits to existing HRL methods.
These benefits include: a systematic means to decompose and to refine task specifications, explicit reasoning over the probabilities of events, the use of planning-based solution techniques (which could incorporate additional problem constraints), and flexibility in the choice of RL algorithm used to learn subsystem policies.
HRL methods use task decompositions to reduce computational complexity, particularly in problems with large state and action spaces \citep{pateria2021hierarchical}.
However, they typically focus on the efficient maximization of discounted reward and they require the meta-policy to be learned; no model of the high-level problem is explicitly constructed.
By contrast, we present a framework that builds a model of the high-level problem with the specific aim of enabling verifiable RL against a rich set of task specifications (\textit{e.g., safely reach a target set with a required probability of success}), while enjoying a similar reduction in sample complexity.

Compositional verification has been studied in formal methods \cite{namautomatic2008, feng2011automated}, but not in the context of RL.
Conversely, recent works have used structured task knowledge to decompose RL problems, however, they do not study how such information can be used to verify RL systems. 
\citet{camacho2017non} and \citet{littman2017environment} both define a task specification language based on linear temporal logic, and subsequently use it to generate reward functions for RL. 
\citet{sarathy2020spotter} incorporates RL with symbolic planning models to learn new operators -- similar to our subtasks -- to aid in the completion of planning objectives. 
Meanwhile, \citet{icarte2018using, toro2019learning, xu2020joint, icarte2020reward} use reward machines, finite-state machines encoding temporally extended tasks in terms of atomic propositions, to break tasks into stages for which separate policies can be learned.  
\citet{neary2020reward} extends the use of reward machines to the multi-agent RL setting, decomposing team tasks into subtasks for individual learners. 
These works all use structured task knowledge to decompose RL problems, however, they do not provide methods for the automated verification and decomposition of task success probabilities, or for the targeted training of subsystems.
\section{Conclusions}

\label{sec:conclusions}

The verification of reinforcement learning (RL) systems is a critical step towards their widespread deployment in engineering applications.
We develop a framework for verifiable and compositional RL in which collections of RL subsystems are composed to achieve an overall task.
We automatically decompose system-level task specifications into individual subtask specifications, and iteratively refine these subtask specifications while training subsystems to satisfy them.
Future directions will study extensions of the framework to multi-level task hierarchies, compositional multi-agent RL systems, and to systems involving partial information.
\clearpage

\fontsize{9.0pt}{11pt} \selectfont

\section{Acknowledgements}
This work was supported in part by ONR N00014-20-1-2115, ARO W911NF-20-1-0140, and NSF 1652113.

\bibliography{bibliography.bib}

\begin{thebibliography}{30}
\providecommand{\natexlab}[1]{#1}

\bibitem[{Baier and Katoen(2008)}]{baier2008principles}
Baier, C.; and Katoen, J.-P. 2008.
\newblock \emph{Principles of model checking}.
\newblock MIT press.

\bibitem[{Barto and Mahadevan(2003)}]{barto2003recent}
Barto, A.~G.; and Mahadevan, S. 2003.
\newblock Recent advances in hierarchical reinforcement learning.
\newblock \emph{Discrete event dynamic systems}, 13(1): 41--77.

\bibitem[{Camacho et~al.(2017)Camacho, Chen, Sanner, and
  McIlraith}]{camacho2017non}
Camacho, A.; Chen, O.; Sanner, S.; and McIlraith, S.~A. 2017.
\newblock Non-markovian rewards expressed in LTL: guiding search via reward
  shaping.
\newblock \emph{10th Annual Symposium on Combinatorial Search}.

\bibitem[{Chevalier-Boisvert, Willems, and Pal(2018)}]{gym_minigrid}
Chevalier-Boisvert, M.; Willems, L.; and Pal, S. 2018.
\newblock Minimalistic Gridworld Environment for OpenAI Gym.
\newblock \url{https://github.com/maximecb/gym-minigrid}.

\bibitem[{Cubuktepe et~al.(2018)Cubuktepe, Jansen, Junges, Katoen, and
  Topcu}]{cubuktepe2018synthesis}
Cubuktepe, M.; Jansen, N.; Junges, S.; Katoen, J.-P.; and Topcu, U. 2018.
\newblock {Synthesis in pMDPs: A Tale of 1001 Parameters}.
\newblock In \emph{International Symposium on Automated Technology for
  Verification and Analysis}, 160--176. Springer.

\bibitem[{Etessami et~al.(2007)Etessami, Kwiatkowska, Vardi, and
  Yannakakis}]{etessami2007multi}
Etessami, K.; Kwiatkowska, M.; Vardi, M.~Y.; and Yannakakis, M. 2007.
\newblock Multi-objective model checking of Markov decision processes.
\newblock In \emph{International Conference on Tools and Algorithms for the
  Construction and Analysis of Systems}, 50--65. Springer.

\bibitem[{Feng, Kwiatkowska, and Parker(2011)}]{feng2011automated}
Feng, L.; Kwiatkowska, M.; and Parker, D. 2011.
\newblock Automated learning of probabilistic assumptions for compositional
  reasoning.
\newblock In \emph{International Conference on Fundamental Approaches to
  Software Engineering}, 2--17. Springer.

\bibitem[{{Gurobi Optimization, LLC}(2021)}]{gurobi}
{Gurobi Optimization, LLC}. 2021.
\newblock Gurobi Optimizer Reference Manual.
\newblock \url{https://www.gurobi.com/documentation/9.1/refman/index.html}.
\newblock Accessed: 2021-12-15.

\bibitem[{Haberfellner et~al.(2019)Haberfellner, Nagel, Becker, B{\"u}chel, and
  von Massow}]{haberfellner2019systems}
Haberfellner, R.; Nagel, P.; Becker, M.; B{\"u}chel, A.; and von Massow, H.
  2019.
\newblock \emph{Systems engineering}.
\newblock Springer.

\bibitem[{Hahn et~al.(2019)Hahn, Perez, Schewe, Somenzi, Trivedi, and
  Wojtczak}]{hahn2019omega}
Hahn, E.~M.; Perez, M.; Schewe, S.; Somenzi, F.; Trivedi, A.; and Wojtczak, D.
  2019.
\newblock Omega-regular objectives in model-free reinforcement learning.
\newblock In \emph{International Conference on Tools and Algorithms for the
  Construction and Analysis of Systems}, 395--412. Springer.

\bibitem[{Juliani et~al.(2018)Juliani, Berges, Teng, Cohen, Harper, Elion, Goy,
  Gao, Henry, Mattar et~al.}]{juliani2018unity}
Juliani, A.; Berges, V.-P.; Teng, E.; Cohen, A.; Harper, J.; Elion, C.; Goy,
  C.; Gao, Y.; Henry, H.; Mattar, M.; et~al. 2018.
\newblock Unity: A general platform for intelligent agents.
\newblock \emph{arXiv preprint arXiv:1809.02627}.

\bibitem[{Junges et~al.(2019)Junges, Abraham, Hensel, Jansen, Katoen, Quatmann,
  and Volk}]{junges2020parameter}
Junges, S.; Abraham, E.; Hensel, C.; Jansen, N.; Katoen, J.-P.; Quatmann, T.;
  and Volk, M. 2019.
\newblock Parameter Synthesis for Markov Models.
\newblock arXiv:1903.07993.

\bibitem[{Kulkarni et~al.(2016)Kulkarni, Narasimhan, Saeedi, and
  Tenenbaum}]{kulkarni2016hierarchical}
Kulkarni, T.~D.; Narasimhan, K.; Saeedi, A.; and Tenenbaum, J. 2016.
\newblock Hierarchical Deep Reinforcement Learning: Integrating Temporal
  Abstraction and Intrinsic Motivation.
\newblock In \emph{Advances in Neural Information Processing Systems},
  volume~29.

\bibitem[{Levy et~al.(2019)Levy, Konidaris, Platt, and
  Saenko}]{levy2017learning}
Levy, A.; Konidaris, G.~D.; Platt, R.~W.; and Saenko, K. 2019.
\newblock Learning Multi-Level Hierarchies with Hindsight.
\newblock In \emph{International Conference on Learning Representations}.

\bibitem[{Littman et~al.(2017)Littman, Topcu, Fu, Isbell, Wen, and
  MacGlashan}]{littman2017environment}
Littman, M.~L.; Topcu, U.; Fu, J.; Isbell, C.; Wen, M.; and MacGlashan, J.
  2017.
\newblock Environment-Independent Task Specifications via GLTL.
\newblock arXiv:1704.04341.

\bibitem[{Nachum et~al.(2018)Nachum, Gu, Lee, and Levine}]{nachum2019data}
Nachum, O.; Gu, S.~S.; Lee, H.; and Levine, S. 2018.
\newblock Data-Efficient Hierarchical Reinforcement Learning.
\newblock In \emph{Advances in Neural Information Processing Systems},
  volume~31.

\bibitem[{Nam, Madhusudan, and Alur(2008)}]{namautomatic2008}
Nam, W.; Madhusudan, P.; and Alur, R. 2008.
\newblock Automatic Symbolic Compositional Verification by Learning
  Assumptions.
\newblock \emph{Formal Methods in System Design}, 32(3): 207–234.

\bibitem[{Neary et~al.(2021)Neary, Xu, Wu, and Topcu}]{neary2020reward}
Neary, C.; Xu, Z.; Wu, B.; and Topcu, U. 2021.
\newblock Reward Machines for Cooperative Multi-Agent Reinforcement Learning.
\newblock In \emph{Proceedings of the 20th International Conference on
  Autonomous Agents and MultiAgent Systems}, AAMAS '21, 934–942.

\bibitem[{Nuseibeh and Easterbrook(2000)}]{nuseibeh2000requirements}
Nuseibeh, B.; and Easterbrook, S. 2000.
\newblock Requirements engineering: a roadmap.
\newblock In \emph{Proceedings of the Conference on the Future of Software
  Engineering}, 35--46.

\bibitem[{Pateria et~al.(2021)Pateria, Subagdja, Tan, and
  Quek}]{pateria2021hierarchical}
Pateria, S.; Subagdja, B.; Tan, A.-h.; and Quek, C. 2021.
\newblock Hierarchical Reinforcement Learning: A Comprehensive Survey.
\newblock \emph{ACM Computing Surveys (CSUR)}, 54(5): 1--35.

\bibitem[{Puterman(2014)}]{puterman2014markov}
Puterman, M.~L. 2014.
\newblock \emph{Markov decision processes: discrete stochastic dynamic
  programming}.
\newblock John Wiley \& Sons.

\bibitem[{Raffin et~al.(2021)Raffin, Hill, Gleave, Kanervisto, Ernestus, and
  Dormann}]{stable-baselines3}
Raffin, A.; Hill, A.; Gleave, A.; Kanervisto, A.; Ernestus, M.; and Dormann, N.
  2021.
\newblock Stable-Baselines3: Reliable Reinforcement Learning Implementations.
\newblock \emph{Journal of Machine Learning Research}, 22(268): 1--8.

\bibitem[{Sarathy et~al.(2021)Sarathy, Kasenberg, Goel, Sinapov, and
  Scheutz}]{sarathy2020spotter}
Sarathy, V.; Kasenberg, D.; Goel, S.; Sinapov, J.; and Scheutz, M. 2021.
\newblock SPOTTER: Extending Symbolic Planning Operators through Targeted
  Reinforcement Learning.
\newblock In \emph{Proceedings of the 20th International Conference on
  Autonomous Agents and MultiAgent Systems}, AAMAS '21, 1118–1126.

\bibitem[{Schulman et~al.(2017)Schulman, Wolski, Dhariwal, Radford, and
  Klimov}]{schulman2017proximal}
Schulman, J.; Wolski, F.; Dhariwal, P.; Radford, A.; and Klimov, O. 2017.
\newblock Proximal Policy Optimization Algorithms.
\newblock arXiv:1707.06347.

\bibitem[{Sutton, Precup, and Singh(1999)}]{sutton1999between}
Sutton, R.~S.; Precup, D.; and Singh, S. 1999.
\newblock Between MDPs and semi-MDPs: A framework for temporal abstraction in
  reinforcement learning.
\newblock \emph{Artificial intelligence}, 112(1-2): 181--211.

\bibitem[{Toro~Icarte et~al.(2018)Toro~Icarte, Klassen, Valenzano, and
  McIlraith}]{icarte2018using}
Toro~Icarte, R.; Klassen, T.; Valenzano, R.; and McIlraith, S. 2018.
\newblock Using reward machines for high-level task specification and
  decomposition in reinforcement learning.
\newblock \emph{International Conference on Machine Learning}, 2107--2116.

\bibitem[{Toro~Icarte et~al.(2022)Toro~Icarte, Klassen, Valenzano, and
  McIlraith}]{icarte2020reward}
Toro~Icarte, R.; Klassen, T.~Q.; Valenzano, R.; and McIlraith, S.~A. 2022.
\newblock Reward machines: Exploiting reward function structure in
  reinforcement learning.
\newblock \emph{Journal of Artificial Intelligence Research}, 73: 173--208.

\bibitem[{Toro~Icarte et~al.(2019)Toro~Icarte, Waldie, Klassen, Valenzano,
  Castro, and McIlraith}]{toro2019learning}
Toro~Icarte, R.; Waldie, E.; Klassen, T.; Valenzano, R.; Castro, M.; and
  McIlraith, S. 2019.
\newblock Learning Reward Machines for Partially Observable Reinforcement
  Learning.
\newblock In \emph{Advances in Neural Information Processing Systems},
  volume~32.

\bibitem[{Vezhnevets et~al.(2017)Vezhnevets, Osindero, Schaul, Heess,
  Jaderberg, Silver, and Kavukcuoglu}]{vezhnevets2017feudal}
Vezhnevets, A.~S.; Osindero, S.; Schaul, T.; Heess, N.; Jaderberg, M.; Silver,
  D.; and Kavukcuoglu, K. 2017.
\newblock Feudal networks for hierarchical reinforcement learning.
\newblock In \emph{International Conference on Machine Learning}, 3540--3549.
  PMLR.

\bibitem[{Xu et~al.(2020)Xu, Gavran, Ahmad, Majumdar, Neider, Topcu, and
  Wu}]{xu2020joint}
Xu, Z.; Gavran, I.; Ahmad, Y.; Majumdar, R.; Neider, D.; Topcu, U.; and Wu, B.
  2020.
\newblock Joint inference of reward machines and policies for reinforcement
  learning.
\newblock \emph{Proceedings of the International Conference on Automated
  Planning and Scheduling}, 30: 590--598.

\end{thebibliography}

\newpage

\onecolumn
\fontsize{11pt}{12pt} \selectfont

\begin{center} 
    \begin{LARGE}
        \textbf{Verifiable and Compositional Reinforcement Learning Systems: \\Supplementary Material}
    \end{LARGE}
\end{center}

\appendix
\section{Proof of Theorem 1}

\newcommand{\hlmTimeHorizon}{N}
\newcommand{\metaDecisionTime}{\tau}
\newcommand{\numMetaDecision}{m}
\newcommand{\numTimeStep}{n}
\newcommand{\reachTrajectories}{\Gamma}
\newcommand{\sigmaAlg}{\Sigma}
\newcommand{\measure}{\mathbb{P}}

In this section, we provide a proof of Theorem \ref{thm:hlm_bounds_true_performance}.

\textbf{Intuition of the Proof. }
While the details of the proof are provided in the remainder of this section, we begin by outlining the intuition behind the proof, which is relatively straightforward. We want to show that if for every subsystem \(\controller \in \controllerSet\) and every entry condition \(\mdpState \in \controllerInitialStateSet_{\controller}\) we have \(\mathbb{P}_{\mdp, \policy_{\controller}}^{\mdpState}(\Diamond_{\leq \timeHorizon_{\controller}} \controllerFinalStateSet_{\controller}) \geq \bernoulliProb_{\controller}\), then \(\mathbb{P}_{\mdp, \abstractPolicy}^{\mdpInitialState}(\Diamond \controllerFinalStateSet_{targ}) \geq \mathbb{P}_{\abstractMDP, \hlmPolicy}^{\abstractInitialState}(\Diamond \abstractSuccessState)\).
We begin by defining \(\mathbb{P}_{\abstractMDP, \hlmPolicy}^{\abstractInitialState}(\Diamond \abstractSuccessState)\) as the sum of the probabilities of every sequence \(\abstractState_0 \controller_0 \abstractState_1 \controller_1 ... \abstractState_{\numMetaDecision}\) of HLM states that eventually reaches the goal state \(\abstractSuccessState\). 
Then, using the theorem's assumption, we observe that for each such sequence there exists a collection of sequences \(\mdpState_0 \controller_0 \mdpAction_0 \mdpState_1 \controller_1 \mdpAction_1 ... \mdpState_{\numTimeStep}\) of environment states, subsystems, and actions that has a higher probability value under the measure \(\mathbb{P}_{\mdp, \abstractPolicy}^{\abstractInitialState}(\cdot)\) than the original sequence had under measure \(\mathbb{P}_{\abstractMDP, \hlmPolicy}^{\abstractInitialState}(\cdot)\).
Finally, we note that every such sequence \(\mdpState_0 \controller_0 \mdpAction_0 \mdpState_1 \controller_1 \mdpAction_1 ... \mdpState_{\numTimeStep}\) of environment states, subsystems, and actions eventually reaches the target set \(\controllerFinalStateSet_{targ}\) and that the aforementioned collections of these sequences are pairwise disjoint. From this, we are able to conclude the result of the theorem.

\textbf{Preliminary Definitions. }
We begin by defining \(\mathbb{P}^{\abstractInitialState}_{\abstractMDP, \hlmPolicy}(\Diamond \abstractSuccessState)\). 
To do so, we define the probability space associated with the high-level model (HLM) \(\abstractMDP\) by following the formalisms laid out in Chapter 10 of \cite{baier2008principles}.
Recall that the HLM \(\abstractMDP = (\abstractStateSet, \abstractInitialState, \abstractSuccessState, \abstractFailureState, \controllerSet, \abstractTransition)\) is a parametric Markov decision process (pMDP) with states \(\abstractStateSet\), action set \(\controllerSet\), and transition probability function \(\abstractTransition\) parametrized by \(\bernoulliProb_{\controller}\) for \(\controller \in \controllerSet\). 
Also, \(\hlmPolicy : \mdpStateSet \times \controllerSet \to [0,1]\) is a stationary policy on \(\abstractMDP\).

Let \(Paths(\abstractMDP, \hlmPolicy, \abstractInitialState)\) denote the set of all possible infinite sequences \(\abstractState_0 \controller_0 \abstractState_1 \controller_1 .. \in (\abstractStateSet \times \controllerSet)^{\omega}\) such that \(\abstractState_0 = \abstractInitialState\) and for every \(i \in \{0,1,...\}\), \(\hlmPolicy(\abstractState_i, \controller_i) > 0\) and \(\abstractTransition(\abstractState_i, \controller_i, \abstractState_{i+1}) > 0\). 
These are the \textit{infinite paths} of state-action pairs that could occur in MDP \(\abstractMDP\) under policy \(\hlmPolicy\), which play the role of the outcomes of our probability space.
Then, define \(Paths_{fin}(\abstractMDP, \hlmPolicy, \abstractInitialState)\) to be the set of all \textit{finite path fragments}, i.e. of all finite sequences \(\abstractState_0 \controller_0 \abstractState_1 \controller_1 ... \abstractState_{\numMetaDecision - 1} \controller_{\numMetaDecision - 1} \abstractState_{\numMetaDecision} \in (\abstractState \times \controllerSet)^*\) such that \(\abstractState_0 = \abstractInitialState\) and for every \(i \in \{0,1,...,\numMetaDecision - 1\}\), \(\hlmPolicy(\abstractState_i, \controller_i) >0\) and \(\abstractTransition(\abstractState_i, \controller_i, \abstractState_{i+1}) > 0\).
The \(\sigma\)-algebra -- denoted by \(\sigmaAlg_{\abstractMDP, \hlmPolicy}^{\mdpInitialState}\) -- which plays the role of the event set of our probability space, is the smallest \(\sigma\)-algebra containing all \textit{cylinder sets}, denoted \(Cyl(\abstractState_0 \controller_0 ... \abstractState_{\numMetaDecision})\), of all finite path fragments \(\abstractState_0 \controller_0 ... \abstractState_{\numMetaDecision} \in Paths_{fin}(\abstractMDP, \hlmPolicy, \abstractInitialState)\). The unique probability measure on the \(\sigma\)-algebra \(\sigmaAlg_{\abstractMDP, \hlmPolicy}^{\abstractInitialState}\) is denoted by \(\measure_{\abstractMDP, \hlmPolicy}^{\abstractInitialState} : \sigmaAlg_{\abstractMDP, \hlmPolicy}^{\abstractInitialState} \to [0,1]\) and is defined in Chapter 10 of \cite{baier2008principles}.

We now define \(\reachTrajectories_{\abstractMDP, \hlmPolicy}^{\abstractState, \abstractSuccessState}\) to be the set of all finite path fragments that reach the goal state \(\abstractSuccessState \in \abstractStateSet\). That is,

\[\reachTrajectories_{\abstractMDP, \hlmPolicy}^{\abstractState, \abstractSuccessState} \defeq Paths_{fin}(\abstractMDP, \hlmPolicy, \abstractState) \cap ((\abstractStateSet \setminus \abstractSuccessState) \times \controllerSet)^* \{\abstractSuccessState\}.\]

Then, since the cylinder sets of all these finite path fragments are pairwise disjoint, we have

\begin{align}
    \mathbb{P}^{\abstractInitialState}_{\abstractMDP, \hlmPolicy}(\Diamond \abstractSuccessState)
    &= \measure_{\abstractMDP, \hlmPolicy}^{\abstractInitialState}(Cyl(\reachTrajectories_{\abstractMDP, \hlmPolicy}^{\abstractInitialState, \abstractSuccessState})) \label{eq:def_hlm_reach_prob}\\
    &= \sum_{\abstractState_0 \controller_0 ...\abstractState_{\numMetaDecision} \in \reachTrajectories_{\abstractMDP, \hlmPolicy}^{\abstractInitialState, \abstractSuccessState}} \measure_{\abstractMDP, \hlmPolicy}^{\abstractInitialState}(Cyl(\abstractState_0 \controller_0 ... \abstractState_{\numMetaDecision})) \label{eq:def_hlm_reach_prob_sum}
\end{align}

In order to define the probability of task success \(\mathbb{P}^{\mdpInitialState}_{\mdp, \abstractPolicy}(\Diamond \controllerFinalStateSet_{targ})\) in MDP \(\mdp\) under meta-policy \(\abstractPolicy\), we may similarly define the infinite paths \(Paths(\mdp, \abstractPolicy, \mdpInitialState)\) to be the set of all infinite sequences \(\mdpState_0 \controller_0 \mdpAction_0 \mdpState_1 \controller_1 \mdpAction_1 ... \in (\mdpStateSet \times \controllerSet \times \mdpActionSet)^{\omega}\) that have positive probability under meta-policy \(\abstractPolicy : \mdpStateSet \times \controllerSet \to [0,1]\) acting in MDP \(\mdp\) from initial state \(\mdpInitialState\).
By similarly defining the finite path fragments \(Paths_{fin}(\mdp, \abstractPolicy, \mdpInitialState)\) we may define the \(\sigma\)-algebra \(\sigmaAlg_{\mdp, \abstractPolicy}^{\mdpInitialState}\) and the corresponding probability measure \(\measure_{\mdp, \abstractPolicy}^{\abstractInitialState}(\cdot)\).

We now define the set \(\reachTrajectories_{\mdp, \abstractPolicy}^{\abstractInitialState, \controllerFinalStateSet_{targ}}\) of all finite path fragments \(\mdpState_0 \controller_0 \mdpAction_0 ... \mdpState_{\numTimeStep - 1} \controller_{\numTimeStep - 1} \mdpAction_{\numTimeStep - 1} \mdpState_{\numTimeStep} \in Paths_{fin}(\mdp, \abstractPolicy, \mdpInitialState)\) such that for every \(t < \numTimeStep\), \(\mdpState_t \in \controllerFinalStateSet_{\controller_{t-1}}\) and \(\controllerFinalStateSet_{\controller_{t-1}} = \controllerFinalStateSet_{targ}\) are not both true, and at time \(t = \numTimeStep\), \(\mdpState_{\numTimeStep} \in \controllerFinalStateSet_{\controller_{\numTimeStep - 1}}\) and \(\controllerFinalStateSet_{\controller_{\numTimeStep - 1}} = \controllerFinalStateSet_{targ}\). 
Then, similarly to as in \eqref{eq:def_hlm_reach_prob}, we have

\begin{equation}
    \mathbb{P}^{\mdpInitialState}_{\mdp, \abstractPolicy}(\Diamond \controllerFinalStateSet_{targ}) = \measure_{\mdp, \abstractPolicy}^{\mdpInitialState}(Cyl(\reachTrajectories_{\mdp, \abstractPolicy}^{\mdpInitialState, \controllerFinalStateSet_{targ}})).
\end{equation}

Given a finite HLM path fragment \(\abstractState_0 \hat{\controller}_0 ... \abstractState_{\numMetaDecision} \in Paths_{fin}(\abstractMDP, \hlmPolicy, \abstractInitialState)\), we define the collection of \textit{compatible environment path fragments} \(\reachTrajectories_{\mdp, \abstractPolicy}(\abstractState_0 \controller_0...\abstractState_{\numMetaDecision})\) to be the set of all finite path fragments \(\mdpState_0 \controller_0 \mdpAction_0 ... \mdpState_{\numTimeStep} \in Paths_{fin}(\mdp, \abstractPolicy, \mdpInitialState)\) such that there exists a collection of \textit{meta-decision times} \(0 = \metaDecisionTime_0 < \metaDecisionTime_1 < ... < \metaDecisionTime_{\numMetaDecision} = \numTimeStep\) with \(\controller_{\metaDecisionTime_{i}} = \controller_{\metaDecisionTime_{i}+1} = ... = \controller_{\metaDecisionTime_{i+1}-1} = \hat{\controller}_i\), and \(\mdpState_{\metaDecisionTime_i}, \mdpState_{\metaDecisionTime_i + 1},...,\mdpState_{\metaDecisionTime_{i+1} - 1} \notin \abstractState_{i+1}\), and \(\mdpState_{\metaDecisionTime_i} \in \abstractState_{i}\) for every \(i \in \{0,1,...,\numMetaDecision\}\) (recall that HLM states \(\abstractState\) correspond to collections of environment states \(\mdpState\)).
Here we use the hat \(\hat{\controller}_i\) to distinguish the \(i^{th}\) subsystem in the HLM path fragment \(\abstractState_0 \hat{\controller}_0 ... \abstractState_{\numMetaDecision}\) from the subsystem \(\controller_i\) of the same index within the environment path fragment \(\mdpState_0 \controller_0 \mdpAction_0 ... \mdpState_{\numTimeStep}\).

\begin{lemma}

Given any finite path fragment \(\abstractState_0 \controller_0 ...\abstractState_{\numMetaDecision} \in Paths_{fin}(\abstractMDP, \hlmPolicy, \abstractInitialState)\) such that \(\abstractState_0, \abstractState_1, ..., \abstractState_{\numMetaDecision} \neq \abstractFailureState\). If, for every subsystem \(\controller \in \controllerSet\) and for every entry condition \(\mdpState \in \controllerInitialStateSet_{\controller}\) we have \(\mathbb{P}_{\mdp, \policy_{\controller}}^{\mdpState}(\Diamond_{\leq \timeHorizon_{\controller}}\controllerFinalStateSet_{\controller}) \geq \bernoulliProb_{\controller}\), then the following inequality holds.

\begin{equation}
    \measure_{\mdp, \abstractPolicy}^{\mdpInitialState}(Cyl(\reachTrajectories_{\mdp, \abstractPolicy}(\abstractState_0 \controller_0 ... \abstractState_{\numMetaDecision}))) \geq \measure_{\abstractMDP, \hlmPolicy}^{\abstractInitialState}(Cyl(\abstractState_0 \controller_0 ... \abstractState_{\numMetaDecision}))
\end{equation}
\end{lemma}
\begin{proof}
This inequality follows from the Theorem's assumption that for every subsystem \(\controller \in \controllerSet\) and for every entry condition \(\mdpState \in \controllerInitialStateSet_{\controller}\), we have \(\successProb_{\policy_{\controller}}^{\controller}(\mdpState) \geq \bernoulliProb_{\controller}\) (recall that \(\successProb_{\policy_{\controller}}^{\controller}(\mdpState) \defeq \mathbb{P}_{\mdp, \policy_{\controller}}^{\mdpState}(\Diamond_{\leq \timeHorizon_{\controller}}\controllerFinalStateSet_{\controller})\)).
Given the finite path fragment \(\abstractState_0 \controller_0 ... \abstractState_{\numMetaDecision} \in Paths_{fin}(\abstractMDP, \hlmPolicy, \abstractInitialState)\) we proceed by induction.

Consider the trivial prefix \(\abstractState_0 \in Paths_{fin}(\abstractMDP, \abstractPolicy, \abstractInitialState)\) of the path fragment.
By definition, we have \(\reachTrajectories_{\mdp, \abstractPolicy}(\abstractState_0) = \{\mdpState_0 \in Paths_{fin}(\mdp, \abstractPolicy, \mdpInitialState) | \mdpState_0 \in \abstractState_0\}\).
Now, by the definitions of \(Paths_{fin}(\abstractMDP, \hlmPolicy, \abstractInitialState)\) and \(Paths_{fin}(\mdp, \abstractPolicy, \mdpInitialState)\), we have \(\abstractState_0 = \abstractInitialState\) and \(\mdpState_0 = \mdpInitialState\).
This implies that \(Cyl(\mdpState_0) = Paths(\mdp, \abstractPolicy, \mdpInitialState)\) and \(Cyl(\abstractState_0) = Paths(\abstractMDP, \hlmPolicy, \abstractInitialState)\) and thus, trivially, \(\measure_{\mdp, \abstractPolicy}^{\mdpInitialState}(Cyl(\mdpState_0)) = \measure_{\abstractMDP, \hlmPolicy}^{\abstractInitialState}(Cyl(\abstractState_0)) = 1\).

Now, for any \(0 \leq l \leq \numMetaDecision - 1\) we consider the prefix \(\abstractState_0 \controller_0 ... \abstractState_{l} \in Paths_{fin}(\abstractMDP, \hlmPolicy, \abstractInitialState)\) of the path fragment \(\abstractState_0 \controller_0 ... \abstractState_l \controller_l ... \abstractState_{\numMetaDecision}\).
Suppose that \(\measure_{\mdp, \abstractPolicy}^{\mdpInitialState}(Cyl(\reachTrajectories_{\mdp, \abstractPolicy}(\abstractState_0 \controller_0 ... \abstractState_l))) \geq \measure_{\abstractMDP, \hlmPolicy}^{\abstractInitialState}(Cyl(\abstractState_0 \controller_0 ... \abstractState_l))\).
We may write the probability of \(Cyl(\reachTrajectories_{\mdp, \abstractPolicy}(\abstractState_0 \controller_0 ... \abstractState_l \controller_l \abstractState_{l+1})))\), corresponding to the prefix of length \(l+1\), in terms of the probability of the prefix of length \(l\) and the probability of all environment path fragments compatible with the HLM transition from \(\abstractState_l\) to \(\abstractState_{l+1}\), as follows: \(\measure_{\mdp, \abstractPolicy}^{\mdpInitialState}(Cyl(\reachTrajectories_{\mdp, \abstractPolicy}(\abstractState_0 \controller_0 ... \abstractState_l \controller_l \abstractState_{l+1}))) = \measure_{\mdp, \abstractPolicy}^{\mdpInitialState}(Cyl(\reachTrajectories_{\mdp, \abstractPolicy}(\abstractState_0 \controller_0 ... \abstractState_l))) \sum_{\mdpState \in \abstractState_l} \alpha(\mdpState) * \abstractPolicy(\mdpState, \controller_l) * \mathbb{P}_{\mdp, \policy_{\controller_l}}^{\mdpState}(\Diamond_{\leq \timeHorizon_{\controller_l}} \abstractState_{l+1})\).
Here, \(\alpha(\mdpState)\) is \textit{some} distribution such that \(\sum_{\mdpState \in \abstractState_l} \alpha(\mdpState) = 1\).
Note that from our definition of the meta-policy \(\abstractPolicy\) in terms of \(\hlmPolicy\),  \(\abstractPolicy(\mdpState, \controller) \defeq \hlmPolicy([\mdpState]_{\eqRelation}, \controller)\), we have \(\abstractPolicy(\mdpState, \controller_{l}) = \hlmPolicy(\abstractState_l, \controller_l)\) for every \(\mdpState \in \abstractState_l\).
Furthermore, as \(\abstractState_{l+1} \neq \abstractFailureState \) by assumption, it must be the case that \(\abstractState_{l+1} = succ(\controller_l)\). This implies, by definition, that \(\controllerFinalStateSet_{\controller_l} \subseteq \abstractState_{l+1}\). Thus, \(\mathbb{P}_{\mdp, \policy_{\controller_l}}^{\mdpState}(\Diamond_{\leq \timeHorizon_{\controller_l}}\abstractState_{l+1}) \geq \mathbb{P}_{\mdp, \policy_{\controller_l}}^{\mdpState}(\Diamond_{\leq \timeHorizon_{\controller_l}} \controllerFinalStateSet_{\controller_l}) \geq \bernoulliProb_{\controller_l}\), where the final inequality follows from the lemma's assumption and from the fact that \(\mdpState \in \controllerInitialStateSet_{\controller_l}\) for every \(\mdpState \in \abstractState_l\).
Putting all this together, we have \(\measure_{\mdp, \abstractPolicy}^{\mdpInitialState}(Cyl(\reachTrajectories_{\mdp, \abstractPolicy}(\abstractState_0 \controller_0 ... \abstractState_l \controller_l \abstractState_{l+1}))) \geq \measure_{\mdp, \abstractPolicy}^{\mdpInitialState}(Cyl(\reachTrajectories_{\mdp, \abstractPolicy}(\abstractState_0 \controller_0 ... \abstractState_l))) * \hlmPolicy(\abstractState_l, \controller_l) * \bernoulliProb_{\controller_l}\)
and given the assumption of our induction step that \(\measure_{\mdp, \abstractPolicy}^{\mdpInitialState}(Cyl(\reachTrajectories_{\mdp, \abstractPolicy}(\abstractState_0 \controller_0 ... \abstractState_l))) \geq \measure_{\abstractMDP, \hlmPolicy}^{\abstractInitialState}(Cyl(\abstractState_0 \controller_0 ... \abstractState_l))\), we obtain the inequality 
\(\measure_{\mdp, \abstractPolicy}^{\mdpInitialState}(Cyl(\reachTrajectories_{\mdp, \abstractPolicy}(\abstractState_0 \controller_0 ... \abstractState_l \controller_l \abstractState_{l+1}))) \geq \measure_{\abstractMDP, \hlmPolicy}^{\abstractInitialState}(Cyl(\abstractState_0 \controller_0 ... \abstractState_l)) * \hlmPolicy(\abstractState_l, \controller_l) * \bernoulliProb_{\controller_l} = \measure_{\abstractMDP, \hlmPolicy}^{\abstractInitialState}(Cyl(\abstractState_0 \controller_0 ... \abstractState_l \controller_l \abstractState_{l+1}))\).
Here, the final equality follows from the definition of the transition probability \(\abstractTransition(\abstractState_{l}, \controller_l, \abstractState_{l+1})\) in terms of the parameter value \(\bernoulliProb_{\controller_{l}}\).
By induction, we conclude the proof.

\label{lemma:prob_bound_single_hlm_path}

\end{proof}

With this lemma in place, we are now ready to prove Theorem \ref{thm:hlm_bounds_true_performance}.

\begin{manualtheorem}{1}
Let \(\controllerSet = \{\controller_1, \controller_2, ..., \controller_{\numControllers}\}\) be a collection of composable subsystems with respect to initial state \(\mdpInitialState\) and target set \(\controllerFinalStateSet_{targ}\) within the environment MDP \(\mdp\). Define \(\abstractMDP\) to be the corresponding HLM and let \(\hlmPolicy\) be a policy in \(\abstractMDP\). If, for every subsystem \(\controller \in \controllerSet\) and for every entry condition \(\mdpState \in \controllerInitialStateSet_{\controller}\), \(\successProb_{\policy_{\controller}}^{\controller}(\mdpState) \geq \bernoulliProb_{\controller}\), then \(\mathbb{P}^{\mdpInitialState}_{\mdp, \abstractPolicy}(\Diamond \controllerFinalStateSet_{targ}) \geq \mathbb{P}^{\abstractInitialState}_{\abstractMDP, \hlmPolicy}(\Diamond \abstractSuccessState)\).
\end{manualtheorem}%

\begin{proof}
Consider any finite path fragment \(\abstractState_0 \controller_0 ...\abstractState_{\numMetaDecision} \in \reachTrajectories_{\abstractMDP, \hlmPolicy}^{\abstractInitialState, \abstractSuccessState}\) that reaches the goal state \(\abstractSuccessState\) in the HLM \(\abstractMDP\).
By Lemma 1, we know that \(\measure_{\mdp, \abstractPolicy}^{\mdpInitialState}(Cyl(\reachTrajectories_{\mdp, \abstractPolicy}(\abstractState_0 \controller_0 ... \abstractState_{\numMetaDecision}))) \geq \measure_{\abstractMDP, \hlmPolicy}^{\abstractInitialState}(Cyl(\abstractState_0 \controller_0 ... \abstractState_{\numMetaDecision}))\), which implies that

\begin{align}
    \sum_{\abstractState_0 \controller_0 ... \abstractState_{\numMetaDecision} \in \reachTrajectories_{\abstractMDP, \hlmPolicy}^{\abstractInitialState, \abstractSuccessState}} \measure_{\mdp, \abstractPolicy}^{\mdpInitialState}(Cyl(\reachTrajectories_{\mdp, \abstractPolicy}(\abstractState_0 \controller_0 ... \abstractState_{\numMetaDecision}))) &\geq \sum_{\abstractState_0 \controller_0 ... \abstractState_{\numMetaDecision} \in \reachTrajectories_{\abstractMDP, \hlmPolicy}^{\abstractInitialState, \abstractSuccessState}} \measure_{\abstractMDP, \hlmPolicy}^{\abstractInitialState}(Cyl(\abstractState_0 \controller_0 ... \abstractState_{\numMetaDecision}))\\
    &= \mathbb{P}_{\abstractMDP, \hlmPolicy}^{\abstractInitialState}(\Diamond \abstractSuccessState).
\end{align}

We note that for any such path fragment \(\abstractState_0 \controller_0 ...\abstractState_{\numMetaDecision} \in \reachTrajectories_{\abstractMDP, \hlmPolicy}^{\abstractInitialState, \abstractSuccessState}\), if follows from the definition of the HLM goal state \(\abstractState_{\numMetaDecision} = \abstractSuccessState\) that \(\reachTrajectories_{\mdp, \abstractPolicy}(\abstractState_0 \controller_0 ... \abstractState_{\numMetaDecision}) \subseteq \reachTrajectories_{\mdp, \abstractPolicy}^{\mdpInitialState, \controllerFinalStateSet_{targ}}\).
Furthermore, the cylinder sets of the compatible environment path fragments \(Cyl(\reachTrajectories_{\mdp, \abstractPolicy}(\abstractState_0 \controller_0 ... \abstractState_{\numMetaDecision}))\) for every HLM path fragment \(\abstractState_0 \controller_0 ...\abstractState_{\numMetaDecision} \in \reachTrajectories_{\abstractMDP, \hlmPolicy}^{\abstractInitialState, \abstractSuccessState}\) are pairwise disjoint. So, we conclude that 

\begin{align}
    \mathbb{P}_{\mdp, \abstractPolicy}^{\mdpInitialState}(\Diamond \controllerFinalStateSet_{targ}) \geq \sum_{\abstractState_0 \controller_0 ... \abstractState_{\numMetaDecision} \in \reachTrajectories_{\abstractMDP, \hlmPolicy}^{\abstractInitialState, \abstractSuccessState}} \measure_{\mdp, \abstractPolicy}^{\mdpInitialState}(Cyl(\reachTrajectories_{\mdp, \abstractPolicy}(\abstractState_0 \controller_0 ... \abstractState_{\numMetaDecision}))) \geq \mathbb{P}_{\abstractMDP, \hlmPolicy}^{\abstractInitialState}(\Diamond \abstractSuccessState).
\end{align}

\end{proof}
\section{Further Experimental Details}

In this section we present details surrounding the training of the individual RL subsystems, additional results plots for the continuous labyrinth experiment, and the results of training the composite system multiple times with different random seeds.
Figure \ref{fig:supp_labyrinth_environments} illustrates both the discrete and continuous labyrinth environment implementations that were used for testing.

\paragraph{Training the RL Subsystems. }
Each RL subsystem is trained to reach its exit conditions, given it begins in one of its entry conditions. 
In the discrete labyrinth environment, these entry and exit conditions correspond to finite collections of states -- locations and orientations within the labyrinth gridworld.
The reward signals used to train the subsystems simply return \(1.0\) when the corresponding exit state(s) have been reached and 0.0 otherwise.
Once an exit condition is reached, the episode ends; it is thus impossible for the subsystems to receive more than a reward of 1.0 per episode of training.
In the continuous labyrinth environment, these entry and exit conditions correspond to compact subsets of points \((x,y,\dot{x}, \dot{y}) \in \mathbb{R}^4\) such that \(\sqrt{(x - x_{\controller})^2 + (y-y_{\controller})^2} \leq 0.5m\) and \(\sqrt{\dot{x}^2 + \dot{y}^2} \leq 0.5\frac{m}{s}\). Here, \((x,y)\) and \((\dot{x}, \dot{y})\) correspond to the position and velocity values, respectively, and \(x_{\controller}\), \(y_{\controller}\) is a pre-specified entry or exit position.
In the continuous environment a more dense reward signal is required for training the subsystems.
If \((x_{\controller}, y_{\controller})\) is the exit position of subsystem \(\controller\), then the subsystem receives a reward of \(-\sqrt{(x - x_{\controller})^2 + (y-y_{\controller})}\) per timestep that it has not reached an exit condition, and it receives a reward of 100 when it reaches an exit condition.
To prevent the agent from learning to purposefully move into the lava in order to end the episode early, we additionally penalize the agent with a reward of -10,000 if it touches the lava.
In both the discrete and the continuous labyrinth environments, if the subsystem collides with any of the lava, then the episode ends and the subsystem has no chance at receiving positive reward in that episode.

To train each RL subsystem, we used the Stable-Baselines3 \cite{stable-baselines3} implementation of the proximal policy optimization (PPO) algorithm \cite{schulman2017proximal} with default parameters.
The values of these algorithm parameters for both the discrete and the continuous labyrinth environments are listed in Table \ref{tab:ppo_params}.

\definecolor{color0}{rgb}{0.12156862745098,0.466666666666667,0.705882352941177}
\definecolor{color1}{rgb}{0.682352941176471,0.780392156862745,0.909803921568627}
\definecolor{color2}{rgb}{1,0.498039215686275,0.0549019607843137}
\definecolor{color3}{rgb}{1,0.733333333333333,0.470588235294118}
\definecolor{color4}{rgb}{0.172549019607843,0.627450980392157,0.172549019607843}
\definecolor{color5}{rgb}{0.596078431372549,0.874509803921569,0.541176470588235}
\definecolor{color6}{rgb}{0.83921568627451,0.152941176470588,0.156862745098039}
\definecolor{color7}{rgb}{1,0.596078431372549,0.588235294117647}
\definecolor{color8}{rgb}{0.580392156862745,0.403921568627451,0.741176470588235}
\definecolor{color9}{rgb}{0.772549019607843,0.690196078431373,0.835294117647059}
\definecolor{color10}{rgb}{0.549019607843137,0.337254901960784,0.294117647058824}
\definecolor{color11}{rgb}{0.768627450980392,0.611764705882353,0.580392156862745}

\definecolor{highlightgrey}{gray}{0.85}

\begin{table}[b]
\centering
 \begin{tabular}{|c | c || c | c || c | c |} 
 \hline
 Learning rate & 
 2.5e-4 &
 Steps per update & 
 512 &
 Minibatch size & 
 64 \\
 \hline\hline
 \pbox{20cm}{Number of epochs when\\optimizing surrogate loss} &
 10 & 
 Discount factor (\(\gamma\)) & 
 0.99 & 
 GAE parameter (\(\lambda\)) & 
 0.95 \\
 \hline\hline
 Clipping parameter
 & 0.2 
 & \pbox{20cm}{Value function\\ coefficient}
 & 0.5 
 & Max gradient norm 
 & 0.5 \\ 
 \hline
\end{tabular}
\vspace*{0.25cm}
\caption{\label{tab:ppo_params} PPO algorithm parameter values used for the training of the RL subsystems.}
\vspace{-0.4cm}
\end{table}

During each loop of Algorithm \ref{alg:ICRL}, a particular subsystem \(\controller\) is selected to train. The selected subsystem is then trained using the PPO algorithm for \(N_{train} = 50,000\) training steps. 
After its training, a new estimate \(\controllerPerformanceEstimate_{\controller}\) of its probability of subtask success is obtained by rolling out the subsystem's learned policy \(300\) separate times from states sampled uniformly from the set of states representing the subsystem's entry conditions.
Each subsystem is given a maximum allowable training budget of \(N_{max}=500,000\) training steps before its most recent estimated performance value \(\controllerPerformanceEstimate_{\controller}\) is added as an upper bound constraint on the corresponding parameter value \(\bernoulliProb_{\controller}\) in the bilinear program \eqref{eq:hlm_opt_objective}-\eqref{eq:hlm_opt_ub_constraints} (as described in lines 10-11 in Algorithm \ref{alg:ICRL}).

\paragraph{Hardware Resources Used. } 
All experiments were run locally on a laptop computer with an Intel i9-11900H 2.5 GHz CPU, an NVIDIA RTX 3060 GPU, and with 16 GB of RAM. 
A complete training run for the discrete gridworld labyrinth, which consists of roughly 1,500,000 total training iterations across all subsystems, takes approximately 25 minutes of wall-clock time.
A complete training run for the continuous labyrinth environment, which consists of roughly 800,000 total environment interactions across all subsystems, takes approximately 60 minutes of wall-clock time.

\begin{figure*}[t!]
    \centering
    \begin{subfigure}[t]{0.45\textwidth}
        \centering \includegraphics[height=5cm]{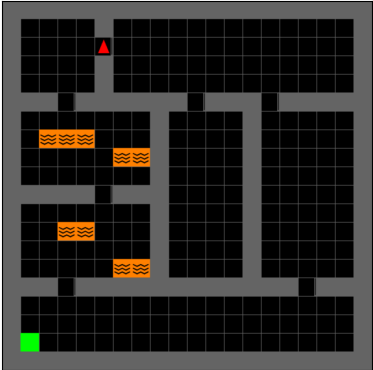}
        \caption{The discrete labyrinth environment.}
        \vspace*{0.5cm}
        \label{fig:supp_discrete_labyrinth_environment}
    \end{subfigure}%
    ~
    \hspace*{\fill}
    \begin{subfigure}[t]{0.45\textwidth}
        \centering \includegraphics[height=5cm]{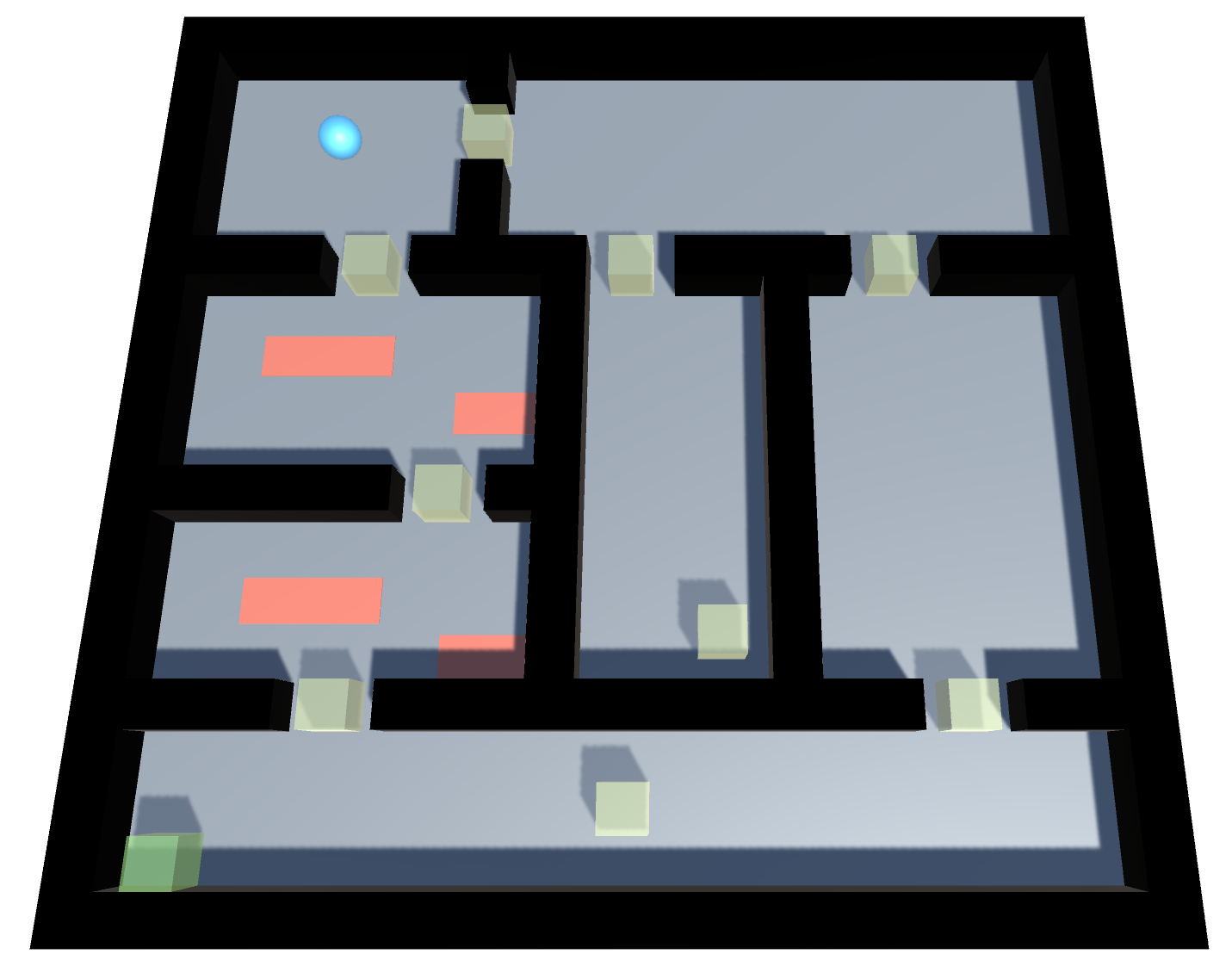}
        \caption{The continuous labyrinth environment.}
        \vspace*{0.5cm}
        \label{fig:supp_continuous_labyrinth_environment}
    \end{subfigure}
    \caption{The labyrinth environments.} 
    \label{fig:supp_labyrinth_environments}
\end{figure*}

\paragraph{Results of Running Experiment with Different Random Seeds. }
To assess the variance in the result of training the system using Algorithm \ref{alg:ICRL}, we ran the presented labyrinth experiment 10 separate times with different random seeds. 
Figure \ref{fig:error_bars} visualizes the results. 
To avoid unnecessary visual clutter, we did not include information pertaining to the individual subsystems in this figure. 
The solid lines indicate the median values across all training runs, while the borders of the shaded region indicate the \(25^{th}\) and \(75^{th}\) percentiles.
We observe high variability in the system's probaility of task success between \(0.6e6\) and \(1.2e6\) total elapsed training steps.
However, we note that by \(1.4e6\) training steps, all runs converge to system behavior that satisfies the overall task specification.
The variance observe during training is mostly due to the subsystem 4, illustrated by dark green in Figure \ref{fig:labyrinth_gridworld}, that is tasked with navigating the top lava room. 
In some instances, after roughly \(0.6e6\) total training steps, the subsystem learns to navigate past the lava with probability of roughly 0.8 and the entire system's performance rises accordingly.
These instances correspond to the top of the shaded region.
Conversely, in some instances subsystem 4 doesn't learn to navigate past the lava within its allowed training budget with any probability of success; these instances correspond to the bottom of the shaded region.
In both cases, subsystem 4 eventually exhausts its training budget without learning to satisfy its subtask specification, in which case alternate subsystems are trained, causing the system's probability of task success to rise above the required threshold after about \(1.4e6\) total training iterations.

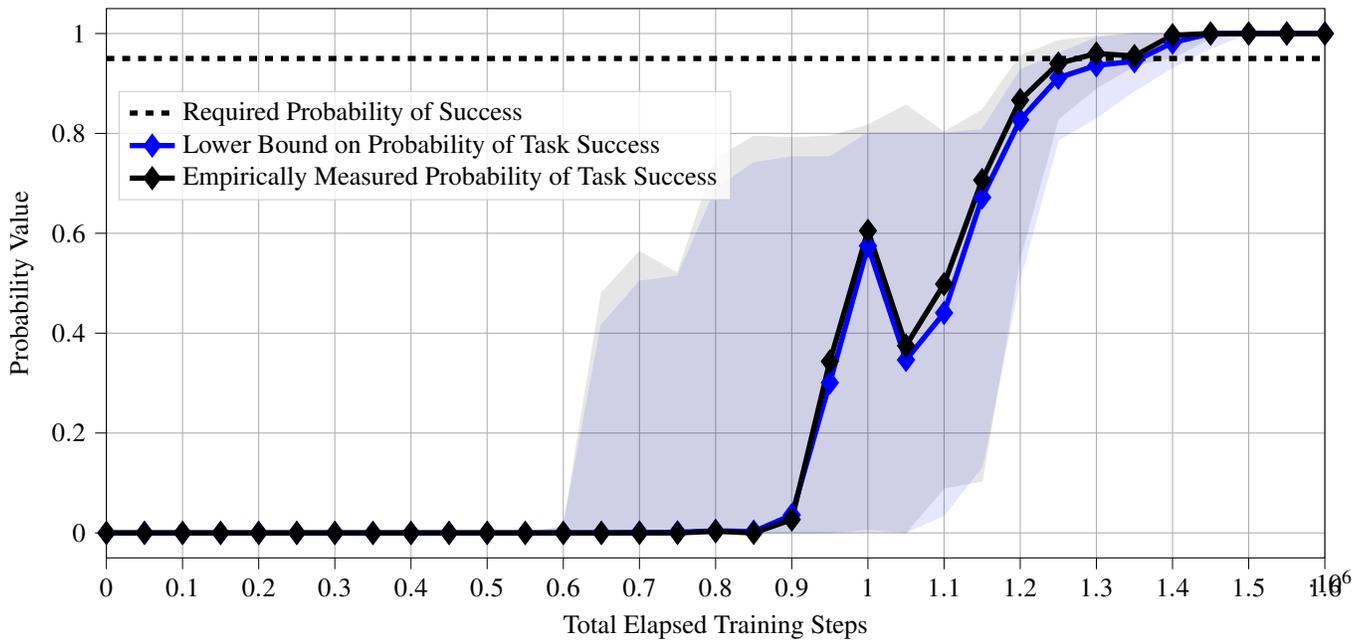
\begin{figure}
    \centering
\begin{tikzpicture}

\begin{axis}[
width=\textwidth,
height=0.5\textwidth,
legend cell align={left},
legend style={
  fill opacity=0.8,
  draw opacity=1,
  text opacity=1,
  at={(0.01, 0.85)},
  anchor=north west,
  draw=white!80!black
},
tick align=outside,
tick pos=left,
x grid style={white!69.0196078431373!black},
xmajorgrids,
xmin=0, xmax=1600000,
xtick style={color=black},
y grid style={white!69.0196078431373!black},
ymajorgrids,
ymin=-0.05, ymax=1.05,
ytick style={color=black},
xlabel={Total Elapsed Training Steps},
ylabel={Probability Value},
every x tick scale label/.style={
    at={(1,0)},xshift=-6.5pt,yshift=-13.5pt,anchor=south west,inner sep=0pt}
]
\path [draw=blue, fill=blue, opacity=0.1]
(axis cs:0,0)
--(axis cs:0,0)
--(axis cs:50000,0)
--(axis cs:100000,0)
--(axis cs:150000,0)
--(axis cs:200000,0)
--(axis cs:250000,0)
--(axis cs:300000,0)
--(axis cs:350000,0)
--(axis cs:400000,0)
--(axis cs:450000,0)
--(axis cs:500000,0)
--(axis cs:550000,0)
--(axis cs:600000,0)
--(axis cs:650000,0)
--(axis cs:700000,0)
--(axis cs:750000,5.48198592592593e-05)
--(axis cs:800000,0.000491630498765432)
--(axis cs:850000,7.4717337037037e-05)
--(axis cs:900000,0.000457878928395062)
--(axis cs:950000,0.000634240883950617)
--(axis cs:1000000,0.00739306376419753)
--(axis cs:1050000,0.00138775649876543)
--(axis cs:1100000,0.0354982714240741)
--(axis cs:1150000,0.132565823812963)
--(axis cs:1200000,0.510187218787037)
--(axis cs:1250000,0.786975925925926)
--(axis cs:1300000,0.83144665695)
--(axis cs:1350000,0.885954586183333)
--(axis cs:1400000,0.932102267777778)
--(axis cs:1450000,0.97042256)
--(axis cs:1500000,1)
--(axis cs:1550000,1)
--(axis cs:1600000,1)
--(axis cs:1600000,1)
--(axis cs:1600000,1)
--(axis cs:1550000,1)
--(axis cs:1500000,1)
--(axis cs:1450000,1)
--(axis cs:1400000,1)
--(axis cs:1350000,1)
--(axis cs:1300000,0.991763426666667)
--(axis cs:1250000,0.959702791666666)
--(axis cs:1200000,0.926842843012346)
--(axis cs:1150000,0.807515037037037)
--(axis cs:1100000,0.800635037037037)
--(axis cs:1050000,0.800635037037037)
--(axis cs:1000000,0.800635037037037)
--(axis cs:950000,0.752925694444444)
--(axis cs:900000,0.752925694444444)
--(axis cs:850000,0.741054138888889)
--(axis cs:800000,0.687890916666667)
--(axis cs:750000,0.513824833333333)
--(axis cs:700000,0.504047314814815)
--(axis cs:650000,0.416246182098766)
--(axis cs:600000,0.011067)
--(axis cs:550000,7.378e-05)
--(axis cs:500000,3.31604938271605e-07)
--(axis cs:450000,0)
--(axis cs:400000,0)
--(axis cs:350000,0)
--(axis cs:300000,0)
--(axis cs:250000,0)
--(axis cs:200000,0)
--(axis cs:150000,0)
--(axis cs:100000,0)
--(axis cs:50000,0)
--(axis cs:0,0)
--cycle;

\path [draw=black, fill=black, opacity=0.1]
(axis cs:0,0)
--(axis cs:0,0)
--(axis cs:50000,0)
--(axis cs:100000,0)
--(axis cs:150000,0)
--(axis cs:200000,0)
--(axis cs:250000,0)
--(axis cs:300000,0)
--(axis cs:350000,0)
--(axis cs:400000,0)
--(axis cs:450000,0)
--(axis cs:500000,0)
--(axis cs:550000,0)
--(axis cs:600000,0)
--(axis cs:650000,0)
--(axis cs:700000,0)
--(axis cs:750000,0)
--(axis cs:800000,0)
--(axis cs:850000,0)
--(axis cs:900000,0)
--(axis cs:950000,0)
--(axis cs:1000000,0.00166666666666667)
--(axis cs:1050000,0.000833333333333333)
--(axis cs:1100000,0.09)
--(axis cs:1150000,0.104166666666667)
--(axis cs:1200000,0.551666666666667)
--(axis cs:1250000,0.829166666666667)
--(axis cs:1300000,0.890833333333333)
--(axis cs:1350000,0.935833333333333)
--(axis cs:1400000,0.951666666666667)
--(axis cs:1450000,0.99)
--(axis cs:1500000,1)
--(axis cs:1550000,1)
--(axis cs:1600000,1)
--(axis cs:1600000,1)
--(axis cs:1600000,1)
--(axis cs:1550000,1)
--(axis cs:1500000,1)
--(axis cs:1450000,1)
--(axis cs:1400000,1)
--(axis cs:1350000,1)
--(axis cs:1300000,0.995)
--(axis cs:1250000,0.985833333333333)
--(axis cs:1200000,0.955)
--(axis cs:1150000,0.846666666666667)
--(axis cs:1100000,0.8025)
--(axis cs:1050000,0.856666666666667)
--(axis cs:1000000,0.816666666666667)
--(axis cs:950000,0.794166666666667)
--(axis cs:900000,0.791666666666667)
--(axis cs:850000,0.794166666666667)
--(axis cs:800000,0.749166666666667)
--(axis cs:750000,0.52)
--(axis cs:700000,0.563333333333333)
--(axis cs:650000,0.48)
--(axis cs:600000,0.0125)
--(axis cs:550000,0)
--(axis cs:500000,0)
--(axis cs:450000,0)
--(axis cs:400000,0.0025)
--(axis cs:350000,0)
--(axis cs:300000,0)
--(axis cs:250000,0)
--(axis cs:200000,0)
--(axis cs:150000,0)
--(axis cs:100000,0)
--(axis cs:50000,0)
--(axis cs:0,0)
--cycle;

\addplot [line width=2pt, black, dashed]
table {%
0 0.95
1600000 0.95
};
\addlegendentry{Required Probability of Success}
\addplot [line width=2pt, blue, mark=diamond*, mark size=3, mark options={solid}]
table {%
0 0
50000 0
100000 0
150000 0
200000 0
250000 0
300000 0
350000 0
400000 0
450000 0
500000 0
550000 0
600000 0
650000 1.24172913580247e-05
700000 0.000886364197530864
750000 0.00115848219753086
800000 0.0041292057037037
850000 0.00265172222222222
900000 0.0358211983333333
950000 0.300859478466667
1000000 0.574904473353086
1050000 0.346476144833333
1100000 0.440576783333333
1150000 0.671817285185185
1200000 0.827085996296296
1250000 0.911347217845679
1300000 0.936313885802469
1350000 0.944366380308642
1400000 0.98187139
1450000 1
1500000 1
1550000 1
1600000 1
};
\addlegendentry{Lower Bound on Probability of Task Success}
\addplot [line width=2pt, black, mark=diamond*, mark size=3, mark options={solid}]
table {%
0 0
50000 0
100000 0
150000 0
200000 0
250000 0
300000 0
350000 0
400000 0
450000 0
500000 0
550000 0
600000 0
650000 0
700000 0
750000 0
800000 0.00333333333333333
850000 0
900000 0.0266666666666667
950000 0.343333333333333
1000000 0.605
1050000 0.375
1100000 0.498333333333333
1150000 0.706666666666667
1200000 0.866666666666667
1250000 0.94
1300000 0.96
1350000 0.955
1400000 0.996666666666667
1450000 1
1500000 1
1550000 1
1600000 1
};
\addlegendentry{Empirically Measured Probability of Task Success}
\end{axis}

\end{tikzpicture}
    \caption{Results of running the discrete gridworld labyrinth experiment with 10 different random seeds. The HLM-predicted probability of task success is plotted in blue, while empirical estimates of the system's probability of task success are plotted in black. The solid line visualizes the median values across all runs, while the borders of the shaded regions visualize the \(25^{th}\) and \(75^{th}\) percentiles.}
    \label{fig:error_bars}
\end{figure}

\paragraph{Additional Results Plots for the Continuous Labyrinth Environment.}
Figures \ref{fig:supp_cont_lab_training_curves} and \ref{fig:supp_cont_lab_training_schedule} illustrate the results of applying the ICRL framework to the continuous labyrinth environment implemented in \textit{Unity} and described in the numerical experiments section of the paper.
In comparison with the results of the discrete labyrinth environment, we note that the total number of elapsed training steps and the exact order of the training for the subsystems is slightly different.
However, the overall result of ICRL's training procedure is very similar between the discrete and continuous versions of the labyrinth environment, despite the marked difference in the environment states, actions, and dynamics.
These similarities help to illustrate the generality of the proposed framework; the ICRL algorithm is indeed agnostic to the specific implementations of the individual RL subsystems.

\begin{figure*}[t!]
    \centering
    \begin{subfigure}[t]{0.45\textwidth}
        \centering 
\begin{tikzpicture}

\definecolor{color0}{rgb}{0.12156862745098,0.466666666666667,0.705882352941177}
\definecolor{color1}{rgb}{0.682352941176471,0.780392156862745,0.909803921568627}
\definecolor{color2}{rgb}{1,0.498039215686275,0.0549019607843137}
\definecolor{color3}{rgb}{1,0.733333333333333,0.470588235294118}
\definecolor{color4}{rgb}{0.172549019607843,0.627450980392157,0.172549019607843}
\definecolor{color5}{rgb}{0.596078431372549,0.874509803921569,0.541176470588235}
\definecolor{color6}{rgb}{0.83921568627451,0.152941176470588,0.156862745098039}
\definecolor{color7}{rgb}{1,0.596078431372549,0.588235294117647}
\definecolor{color8}{rgb}{0.580392156862745,0.403921568627451,0.741176470588235}
\definecolor{color9}{rgb}{0.772549019607843,0.690196078431373,0.835294117647059}
\definecolor{color10}{rgb}{0.549019607843137,0.337254901960784,0.294117647058824}
\definecolor{color11}{rgb}{0.768627450980392,0.611764705882353,0.580392156862745}

\begin{axis}[
tick align=inside,
tick pos=left,
x grid style={white!69.0196078431373!black},
  ticklabel style = {font=\footnotesize},
xmajorgrids,
xmin=0, xmax=750000,
xtick style={color=black},
y grid style={white!69.0196078431373!black},
ymajorgrids,
ymin=-0.05, ymax=1.05,
height=4.0cm,
width=0.99\textwidth,
ytick={0.0, 0.2, 0.4, 0.6, 0.8, 1.0},
yticklabels={0.0, 0.2, 0.4, 0.6, 0.8, 1.0},
xtick={0, 100000, 200000, 300000, 400000, 500000, 600000, 800000, 1000000, 1200000, 1400000},
xticklabels={0.0, 1.0, 2.0, 3.0, 4.0, 5.0, 6.0, 8.0, 10.0, 12.0, },
ytick style={color=black},
legend to name=named,
every x tick scale label/.style={
    at={(1,0)},xshift=-15.5pt,yshift=-10.0pt,anchor=south west,inner sep=0pt},
ylabel={Probability Value},
xlabel={Elapsed Total Training Steps}
]
\addplot [very thick, color0, mark=diamond*, mark size=2, mark options={solid}, forget plot]
table {%
0 0
0 0
50000 0.998
100000 0.998
150000 0.998
200000 0.998
250000 0.998
300000 0.998
350000 0.998
400000 0.998
450000 0.998
500000 0.998
550000 0.998
600000 0.998
650000 0.998
700000 0.998
750000 0.998
};
\addplot [very thick, color1, mark=diamond*, mark size=2, mark options={solid}, forget plot]
table {%
0 0
0 0
50000 0
100000 0
150000 0
200000 0
250000 0
300000 0
350000 0
400000 0
450000 0
500000 0
550000 0
600000 1
650000 1
700000 1
750000 1
};
\addplot [very thick, color2, mark=diamond*, mark size=2, mark options={solid}, forget plot]
table {%
0 0
0 0
50000 0
100000 0
150000 0
200000 0
250000 0
300000 0
350000 0
400000 0
450000 0
500000 0
550000 0
600000 0
650000 0
700000 0
750000 0
};
\addplot [very thick, color3, mark=diamond*, mark size=2, mark options={solid}, forget plot]
table {%
0 0
0 0
50000 0
100000 0
150000 0
200000 0
250000 0
300000 0
350000 0
400000 0
450000 1
500000 1
550000 1
600000 1
650000 1
700000 1
750000 1
};
\addplot [very thick, color4, mark=diamond*, mark size=2, mark options={solid}, forget plot]
table {%
0 0
0 0
50000 0
100000 0
150000 0
200000 0
250000 0
300000 0
350000 0
400000 0
450000 0
500000 0
550000 0
600000 0
650000 0
700000 0
750000 0
};
\addplot [very thick, color5, mark=diamond*, mark size=2, mark options={solid}, forget plot]
table {%
0 0
0 0
50000 0
100000 0
150000 0
200000 0
250000 0
300000 0
350000 0
400000 0
450000 0
500000 0
550000 0
600000 0
650000 0
700000 0
750000 0
};
\addplot [very thick, color6, mark=diamond*, mark size=2, mark options={solid}, forget plot]
table {%
0 0
0 0
50000 0
100000 0
150000 0
200000 0
250000 0
300000 0
350000 0
400000 0
450000 0
500000 0
550000 0
600000 0
650000 0
700000 0
750000 0
};
\addplot [very thick, color7, mark=diamond*, mark size=2, mark options={solid}, forget plot]
table {%
0 0
0 0
50000 0
100000 0
150000 0
200000 0
250000 0
300000 0
350000 0
400000 0
450000 0
500000 0
550000 0
600000 0
650000 0
700000 0
750000 0
};
\addplot [very thick, color8, mark=diamond*, mark size=2, mark options={solid}, forget plot]
table {%
0 0
0 0
50000 0
100000 0
150000 0
200000 0
250000 0
300000 0
350000 0
400000 0
450000 0
500000 1
550000 1
600000 1
650000 1
700000 1
750000 1
};
\addplot [very thick, color9, mark=diamond*, mark size=2, mark options={solid}, forget plot]
table {%
0 0
0 0
50000 0
100000 0
150000 0
200000 0
250000 0
300000 0
350000 0
400000 0
450000 0
500000 0
550000 0
600000 0
650000 0
700000 0
750000 0
};
\addplot [very thick, color10, mark=diamond*, mark size=2, mark options={solid}, forget plot]
table {%
0 0
0 0
50000 0
100000 0
150000 0
200000 0
250000 0
300000 0
350000 1
400000 1
450000 1
500000 1
550000 1
600000 1
650000 1
700000 1
750000 1
};
\addplot [very thick, color11, mark=diamond*, mark size=2, mark options={solid}, forget plot]
table {%
0 0
0 0
50000 0
100000 0
150000 0
200000 0
250000 0
300000 0
350000 0
400000 0
450000 0
500000 0
550000 0.565
600000 0.565
650000 0.835
700000 0.923
750000 0.986
};
\addplot [line width=2pt, black, dashed]
table {%
0 0.95
750000 0.95
};
\addlegendentry{Required Probability of Success}
\addplot [line width=2pt, blue, mark=diamond*, mark size=3, mark options={solid}]
table {%
0 0
0 0
50000 0
100000 0
150000 0
200000 0
250000 0
300000 0
350000 0
400000 0
450000 0
500000 0
550000 0
600000 0.565
650000 0.835
700000 0.923
750000 0.986
};
\addlegendentry{Lower Bound on Probability of Task Success}
\addplot [line width=2pt, black, mark=diamond*, mark size=3, mark options={solid}]
table {%
0 0
0 0
50000 0
100000 0
150000 0
200000 0
250000 0
300000 0
350000 0
400000 0
450000 0
500000 0
550000 0
600000 0.641
650000 1
700000 1
750000 1
};
\addlegendentry{Empirically Measured Probability of Task Success}
\addplot [line width=4pt, red, dashed, forget plot]
table {%
250000 -0.05
250000 1.05
};
\end{axis}

\end{tikzpicture}
        \caption{Estimated task and subtask success probabilities during training.}
        \vspace*{0.5cm}
        \label{fig:supp_cont_lab_training_curves}
    \end{subfigure}%
    ~
    \hspace*{\fill}
    \begin{subfigure}[t]{0.45\textwidth}
        \centering 
\begin{tikzpicture}

\definecolor{color0}{rgb}{0.12156862745098,0.466666666666667,0.705882352941177}
\definecolor{color1}{rgb}{0.596078431372549,0.874509803921569,0.541176470588235}
\definecolor{color2}{rgb}{0.549019607843137,0.337254901960784,0.294117647058824}
\definecolor{color3}{rgb}{1,0.733333333333333,0.470588235294118}
\definecolor{color4}{rgb}{0.580392156862745,0.403921568627451,0.741176470588235}
\definecolor{color5}{rgb}{0.768627450980392,0.611764705882353,0.580392156862745}
\definecolor{color6}{rgb}{0.682352941176471,0.780392156862745,0.909803921568627}

\begin{axis}[
height=4.0cm,
width=0.99\textwidth,
tick align=inside,
tick pos=left,
x grid style={white!69.0196078431373!black},
ticklabel style = {font=\footnotesize},
xmajorgrids,
xmin=0, xmax=750000,
xtick style={color=black},
y grid style={white!69.0196078431373!black},
ymajorgrids,
ymin=-0.55, ymax=11.55,
ytick style={color=black},
ytick={0,1,2,3,4,5,6,7,8,9,10,11},
yticklabels = {0,1,2,3,4,5,6,7,8,9,10,11},
xtick={0, 100000, 200000, 300000, 400000, 500000, 600000, 800000, 1000000, 1200000, 1400000},
xticklabels={0.0, 1.0, 2.0, 3.0, 4.0, 5.0, 6.0, 8.0, 10.0, 12.0, },
every x tick scale label/.style={
    at={(1,0)},xshift=-15.5pt,yshift=-10.0pt,anchor=south west,inner sep=0pt},
ylabel={Subsystem Index},
xlabel={Elapsed Total Training Steps}
]
\addplot [line width=8pt, color0]
table {%
0 0
50000 0
};
\addplot [line width=8pt, color1]
table {%
50000 4
100000 4
};
\addplot [line width=8pt, color1]
table {%
100000 4
150000 4
};
\addplot [line width=8pt, color1]
table {%
150000 4
200000 4
};
\addplot [line width=8pt, color1]
table {%
200000 4
250000 4
};
\addplot [line width=8pt, color2]
table {%
250000 10
300000 10
};
\addplot [line width=8pt, color2]
table {%
300000 10
350000 10
};
\addplot [line width=8pt, color3]
table {%
350000 3
400000 3
};
\addplot [line width=8pt, color3]
table {%
400000 3
450000 3
};
\addplot [line width=8pt, color4]
table {%
450000 8
500000 8
};
\addplot [line width=8pt, color5]
table {%
500000 11
550000 11
};
\addplot [line width=8pt, color6]
table {%
550000 1
600000 1
};
\addplot [line width=8pt, color5]
table {%
600000 11
650000 11
};
\addplot [line width=8pt, color5]
table {%
650000 11
700000 11
};
\addplot [line width=8pt, color5]
table {%
700000 11
750000 11
};
\addplot [line width=5.6pt, red, dashed]
table {%
250000 -0.55
250000 11.55
};
\end{axis}

\end{tikzpicture}
        \caption{Automatically generated subsystem training schedule.}
        \vspace*{0.5cm}
        \label{fig:supp_cont_lab_training_schedule}
    \end{subfigure}
    \caption{
        Results for the continuous labyrinth environment. 
        Each subtask is represented by a different color, matching those used in Figure \ref{fig:labyrinth_gridworld}.
        Similarly to as with the discrete labyrinth environment, ICRL initially attempts to train the subsystems \(\controller_0, \controller_4, \controller_5,\) and \(\controller_9\), which are used to move straight down through the rooms containing lava to reach the goal. 
        The dotted red line illustrates the point in training at which the ICRL algorithm automatically refines the subtask specifications, resulting instead in the training of the subsystems \(\controller_1, \controller_3, \controller_8, \controller_{10}\), and \(\controller_{11}\).
    } 
\end{figure*}
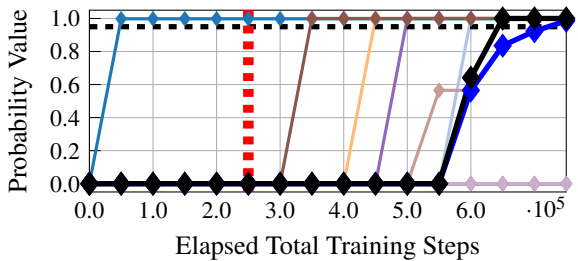
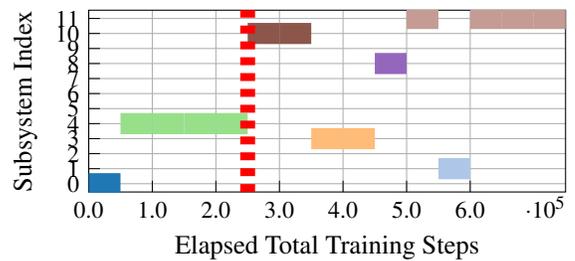

\end{document}